\documentclass[a4paper, 1p, oneside, onecolumn, preprint]{elsarticle}
\usepackage{microtype}
\usepackage{graphicx}
\usepackage{subfigure}
\usepackage{booktabs} 
\usepackage{lineno}
\modulolinenumbers[5]
\usepackage{hyperref}
\usepackage{algorithmic}
\usepackage[tbtags]{amsmath}
\usepackage{amssymb}
\usepackage{amsthm}
\usepackage{bbm}
\usepackage{xcolor}
\usepackage{algorithm}
\usepackage[OT1]{fontenc} 
\usepackage{xr}
\usepackage{tabulary,booktabs}
\usepackage{multirow}
\usepackage{bm}
\newcolumntype{M}[1]{>{\centering\arraybackslash}m{#1}}

\def\reals{\mathbb{R}}
\def\E{\mathbb{E}}
\def\v{\text{vec}}
\def\x{\bm{x}}

\def\W{\bm{W}}
\def\X{\bm{X}}
\def\Y{\bm{Y}}
\def\O{\mathcal{O}}
\def\I{\bm{I}}
\def\B{\bm{B}}
\def\A{\bm{A}}
\def\E{\bm{E}}
\def\U{\bm{U}}
\def\H{\bm{H}}
\def\y{\bm{y}}
\def\M{\bm{M}}
\def\w{\bm{\theta}}
\def\w{\bm{w}}
\def\0{\bm{0}}
\def\G{\mathbf{G}}
\def\varphib{\bm{\varphi}}
\def\psib{\bm{\psi}}
\def\xib{\bm{\xi}}
\def\phib{\bm{\phi}}
\def\iotab{\bm{\iota}}

\newtheorem{lemma}{Lemma}

\newtheorem{theorem}{Theorem}

\begin{document}
\title{A Convergence Analysis of Nesterov's Accelerated Gradient Method  in Training Deep Linear Neural Networks}
\author[1]{Xin Liu}\ead{1036870846@qq.com}
\author[2]{Wei Tao} \ead{wtao_plaust@163.com}
\author[1]{Zhisong Pan\corref{cor1}}\ead{hotpzs@hotmail.com}
\tnotetext[t1]{This work was supported by National Natural Science Foundation of China (No.62076251 and No.62106281).}
\cortext[cor1]{Corresponding author}

\affiliation[1]{organization={Command $\&$ Control Engineering College, Army Engineering University of PLA}, 
                 postcode={210007}, 
                 city={Nanjing}, 
                 country={P.R. China.}}
\affiliation[2]{organization={Center for Strategic Assessment and Consulting, Academy of Military Science}, 
                 postcode={100091}, 
                 city={Beijing}, 
                 country={P.R. China.}}

\begin{abstract}
Momentum methods, including heavy-ball~(HB) and Nesterov's accelerated gradient~(NAG), are widely used in training neural networks for their fast convergence.
However, there is a lack of theoretical guarantees for their convergence and acceleration since the optimization landscape of the neural network is non-convex.
Nowadays, some works make progress towards understanding the convergence of momentum methods in an over-parameterized regime, where the number of the parameters exceeds that of the training instances.
Nonetheless, current results mainly focus on the two-layer neural network, which are far from explaining the remarkable success of the momentum methods in training deep neural networks.
Motivated by this, we investigate the convergence of NAG with constant learning rate and  momentum parameter in training two architectures of deep linear networks: deep fully-connected linear neural networks and deep linear ResNets.
Based on the over-parameterization regime, we first analyze the residual dynamics induced by the training trajectory of NAG for a deep fully-connected linear neural network under the random Gaussian initialization.
Our results show that NAG can converge to the global minimum at a $(1 - \O(1/\sqrt{\kappa}))^t$ rate, where $t$ is the iteration number and $\kappa > 1$ is a constant depending on the condition number of the feature matrix.
Compared to the $(1 - \O(1/{\kappa}))^t$ rate of GD, NAG achieves an acceleration over GD.
To the best of our knowledge, this is the first theoretical guarantee for the convergence of NAG to the global minimum in training deep neural networks.
Furthermore, we extend our analysis to deep linear ResNets and derive a similar convergence result.
\end{abstract}
\begin{keyword}
Deep linear neural network \sep  Over-parameterization \sep Nesterov's accelerated gradient method  
\end{keyword}
\maketitle

\section{Introduction}
Deep learning has achieved great empirical success in various areas, such as image classification~\cite{KrizhevskySH12}, natural language processing~\cite{OtterMK21} and game playing~\cite{SilverHMGSDSAPL16}.
In practice, they often involve networks with stacked layers, whose depth vary from 16~\cite{SimonyanZ14a} to 96~\cite{brown2020language} and even more.
Typically, deep neural networks are trained using first-order methods, which only exploit the objective values and gradients.
Gradient descent~(GD) is the most well-known first-order method, whose history can be dated back to the XIX century~\cite{cauchy1847methode}.
Later on, various variants of GD have been developed by adding momentums to improve its performance, such as heavy-ball~(HB)~\cite{P64} and Nesterov's accelerated gradient~(NAG)~\cite{nesterov1983method}.
In practice, NAG is widely used in training neural networks and attains faster convergence over GD~\cite{SMDH13, Schmidt2021}.
Moreover, it has been the default momentum scheme implementing in many popular deep learning libraries such as PyTorch~\cite{DBLP:conf/nips/PaszkeGMLBCKLGA19}, Keras~\cite{gulli2017deep} and TensorFlow~\cite{DBLP:conf/osdi/AbadiBCCDDDGIIK16}.

Despite the fact that the optimization problem for training neural networks is non-convex, first-order methods are capable of achieving near-zero training loss~\cite{ZBHRV17}. 
However, there is a lack of theoretical guarantees for gradient-based methods to find the global minimum for non-convex problems~\cite{DBLP:journals/mp/MurtyK87}.
Recently, some progress has been achieved via analyzing the optimization landscape of the neural network~\cite{K16, HM16,LK17,ZL18,LvB18}.
But these works do not provide the convergence results.
In addition to the landscape analysis, some works aim at studying the training trajectory of GD for a two-layer neural network in an over-parameterized regime~\cite{ACGH19,LY17,ADHLSW19_icml}, where the number of parameters is larger than that of training instances.
\cite{DH19, HXP20} further showed that GD converges to the global minimum for training deep linear fully-connected neural networks with different initialization schemes.
Nonetheless, these results are limited to GD, and momentum methods are rarely explored except three works~\cite{bu2020dynamical,DBLP:journals/corr/abs-2107-01832,wang2020provable}.
Wang~\textit{et al.}~\cite{wang2020provable} showed that HB converges to the global minimum in training over-parameterized neural networks.
The other two works~\cite{bu2020dynamical,DBLP:journals/corr/abs-2107-01832} provided the convergence results of NAG on the two-layer neural network, which is far from explaining the remarkable success of NAG in training deep neural networks.

In this work, we extend and generalize the existing analysis of NAG in training the two-layer neural network to deep linear networks, including deep fully-connected linear neural networks and deep linear ResNets.
Although deep linear networks have simple frameworks, their optimization landscapes are high-dimensional and non-convex.
Meanwhile, the deep linear neural has a layered structure.
These features are similar to deep non-linear networks, leading to increased interest in characterizing their properties.
As a result, analyzing the deep linear network will be helpful in providing insights into understanding deep non-linear networks.
Our work is inspired by recent advances in deep linear neural networks~\cite{DH19,wang2020provable, DBLP:conf/iclr/ZouLG20}.
The main technical challenge lies in analyzing the residual dynamics induced by the training trajectory of NAG.
Our contribution can be summarized as follows:
\begin{enumerate}
	\item 	We first establish the convergence of NAG in training an $L$-layer fully-connected linear neural network under the random Gaussian initialization.
	Specifically, utilizing the gram matrix defined on the feature matrix and the parameters of the network, we derive the corresponding residual dynamics.
	When the width of the neural network satisfies $m = \tilde{\Omega}(L)$\footnote{{We omit the dependence on other parameters here. The details of the requirement is referred to Theorem~\ref{thm:LinearNet}}}, with high probability, we show that the residual error of NAG can reach zero at a $(1-\O(1/\sqrt{\kappa}))^t$ rate, where $t$ is the iteration number and $\kappa = \O(\|\X\|^2/\sigma_{min}^2(\X))>1$ ($\X$ is the feature matrix).
	Compared to the convergence rate $(1-\O(1/{\kappa}))^t$ of GD~\cite{DH19}, NAG achieves an acceleration over GD.
	\item Based on the same analysis framework, we extend the convergence result to the deep linear ResNet.
	 We demonstrate that NAG can achieve convergence to the global minimum at a similar rate as the result of the deep fully-connected linear neural network, where the corresponding convergecne rate is also faster than that of GD as proved in~\cite{DBLP:conf/iclr/ZouLG20}.
	Moreover, the requirement of the width $m$ for the deep linear ResNet has no dependence on depth $L$.
\end{enumerate}
To the best of our knowledge, this is the first theoretical convergence and acceleration guarantee for NAG in training deep linear neural networks, which may shed light on understanding the optimization behavior of NAG for deep non-linear neural networks.

\section{Related works} \label{sec: rel}

\subsection{Momentum methods}
Momentum methods date back to the seminal work by Polyak~\cite{P64}, in which the HB method was proposed. 
When the objective function is  twice differentiable, strongly convex and smooth, they proved HB converges to the global minimum at an asymptotic linear rate.
By blending  gradients and iterates, Nesterov~\cite{nesterov1983method} proposed the NAG method for smooth convex problems, which attains the accelerated convergence rate $\mathcal{O}(1/t^2)$ compared to the rate $\mathcal{O}(1/t)$ of GD.

In the non-convex regime, some works provided convergence results of momentum methods in terms of the first-order stationary point instead of global minimum~\cite{chen2018convergence,carmon2017convex}.
Recently,  remarkable progress has been achieved in deriving the global convergence of momentum methods in the over-parameterization regime.
Wang \textit{et al.}~\cite{wang2020provable} provided the convergence of HB in training a two-layer neural network and a deep linear network under the identity initialization.
They proved that HB linearly converges to a global minimum at a faster rate than that of GD.
Liu \textit{et al.}~\cite{DBLP:journals/corr/abs-2107-01832} established the convergence result of NAG in training a two-layer ReLU neural network.
From a continuous view, Bu \textit{et al.}~\cite{bu2020dynamical} derived the global convergence of HB and NAG by exploiting the approximation between optimizers with infinitesimal learning rate and ordinary differential equation.
However, the above results of NAG are all limited to the two-layer neural network.

\subsection{Deep linear networks}
Recent works investigated the loss landscape of deep linear networks\cite{K16, HM16,LK17, ZL18,DBLP:conf/iclr/0001B17}, which showed that all local minimums are global minimums with certain assumptions.
Their results only provide the existence of the global minimum,  but can not explain why neural networks trained by gradient-based methods are capable to attain the global minimum as shown in~\cite{ZBHRV17}.

Based on the over-parameterization assumption, a series of works studied the convergence of gradient-based methods in training neural networks.
Du \textit{et al.}~\cite{DH19} showed that GD converges to the global minimum for training a deep linear network with the random Gaussian initialization. 
The requirement of the width of the hidden layers  depends linearly on the depth of the network.
Hu \textit{et al.}~\cite{HXP20} extended the analysis to deep linear networks with the orthogonal initialization, and its requirement of the width is independent of the depth.
Wang \textit{et al.}~\cite{wang2020provable} generalized the above result to HB, and provided theoretical guarantee for the acceleration of HB over GD.
Zou \textit{et al.}~\cite{DBLP:conf/iclr/ZouLG20} focused on the deep residual linear network and provided the convergence of GD and stochastic GD.

\section{Preliminaries} \label{sec: pro}

\subsection{Notations}
We use lowercase, lowercase boldface and uppercase boldface letters to represent  scalars, vectors and matrices, respectively.
We denote $\W^{j:i} = \prod_{l=i}^j \W^l$ for $1\leq i\leq j$ and $\W^{i-1:i} = \I$.
In addition, we denote $\I_n$ and $\mathbf{0}_n$ as the identity matrix and zero matrix with $n \times n$ dimension.
We use $\|\cdot\|$ as the $\ell_2$ norm of a vector or the spectral norm of a matrix, and use $\|\cdot\|_F$ as the Frobenius norm of a matrix.
We denote $\v(A)$ as the vectorization of a matrix $A$ in column-first order.
We denote $\otimes$ as the Kronecker product.
We use the standard $\mathcal{O}(\cdot)$, $\Omega(\cdot)$ and $\Theta(\cdot)$ asymptotic notations for hiding constant factors.

\subsection{Problem Setup}
In this paper, we consider an empirical risk minimization problem with the square loss
\begin{eqnarray}
\label{objective}
	\min_{\w} \ell(\w) = \frac{1}{2}\sum_{i=1}^n(f_{\w}(\x_i) - \mathbf{y}_i)^2,
\end{eqnarray}
where $\w$ is the parameter of the model $f$, $\x_i \in \reals^{d_x}$ and $\y_i\in \reals^{d_y}$ denote the feature and label of the $i-$th training instance, respectively.

GD is the most widely used method for optimizing the model $f$, it follows
\begin{eqnarray}
\label{procudure: GD}
\w_{t+1} = \w_t - \eta \nabla_{\w} \ell(\w_t),
\end{eqnarray}
where $\eta>0$ is the learning rate and $t$ is the iteration number.
To accelerate the convergence, NAG was proposed by combining the history of gradients into the current learning procedure as
\begin{eqnarray}
\label{procudure: NAG_1}
\bm{v}_{t+1} &=& \w_t -  \eta \nabla_{\w} \ell(\w_t) \nonumber\\
\w_{t+1} &=& \bm{v}_{t+1} + \beta(\bm{v}_{t+1} - \bm{v}_t)g,
\end{eqnarray}
where $0\leq\beta\leq 1$ is the momentum parameter.
The model is initialized as $\bm{v}_0 = \w_0 = \w_{-1}$.
Meanwhile, NAG has an equivalent form as
\begin{eqnarray}
\label{procudure: NAG_2}
	\bm{M}_t &=& \beta \bm{M}_{t-1} -\eta \beta \left(\nabla_{\w} \ell(\w_t) - \nabla_{\w} \ell(\w_{t-1})\right)-\eta\nabla_{\w} \ell(\w_t) \nonumber\\
	\w_{t+1} &=& \w_t + \bm{M}_t,
\end{eqnarray}
where $\bm{M}$ is the momentum term with initialized value $\bm{M}_{-1} = 0$.

In this work, we consider the following two architectures of deep linear networks.
\begin{itemize}
	\item \textbf{Deep fully-connected linear neural network:}
	Following the work~\cite{DH19}, we consider a fully-connected linear neural network with $L$ hidden layers 
\begin{equation}
\label{model: linear}
	f_{\W}(\x) = \frac{1}{\sqrt{m^{L-1} d_{y}}}\W^L\cdots \W^1 \x,
\end{equation}
where $\W^1 \in \reals^{m \times d_{x}}$, $\W^2, \cdots, \W^{L-1 }\in \reals^{m \times m}$ and $\W^L \in \reals^{d_{y} \times m}$ are the parameters of each layer.
All the elements of $\W^1, \cdots, \W^L$ are i.i.d initialized with standard Gaussian distribution $\mathcal{N}(0, 1)$.
This initialization scheme is also known as the Xavier initialization~\cite{GlorotB10}.
Note that $\frac{1}{\sqrt{m^{L-1} d_{y}}}$ is a scaling factor according to~\cite{DH19}.

\item \textbf{Deep linear ResNet:}
Deep ResNet architecture was proposed by He \textit{et al.}~\cite{Rnet16}, which applies the residual links to enable gradient-based methods to optimize deeper networks.
	For the deep linear ResNet, we consider the following architecture as studied in~\cite{DBLP:conf/iclr/ZouLG20}
\begin{equation}
\label{model: resnets}
	f_{\W}(\x) = \B(\I+\W^L)\cdots(\I+ \W^1)\A \x,
\end{equation}
where $\{\W^i \in \reals^{m \times m}, \; i \in [L]\}$ denotes the hidden layers, $\A \in \reals^{m \times d_x}$ and $\B \in \reals^{d_y \times m}$ denote the input and output layers, respectively.
We adopt the initialization scheme as~\cite{DBLP:conf/iclr/ZouLG20}, in which the hidden layers are initialized with zero matrices, the initialization of the input and output layers uses $\mathcal{N}(0, 1)$.
In addition, we follow the settings in~\cite{DBLP:conf/iclr/ZouLG20}
that only train the hidden layers and keep $\A$ and $\B$ fixed during training.
\end{itemize}

\section{Theoretical results}
In this section, we introduce the main convergence results of NAG.
To start with, we briefly state the procedures of our proof.
\begin{enumerate}
	\item Firstly, we establish the residual dynamics of NAG as $\begin{bmatrix}
\xib_{t+1} \\
\xib_{t} 
\end{bmatrix}
= 
\G
\begin{bmatrix}
\xib_{t} \\
\xib_{t-1} 
\end{bmatrix}
+
\begin{bmatrix}
\varphib_t \\ \bm{0}
\end{bmatrix}$, where $\xib_{t}$ denotes the residual error at iteration $t$, $\G$ is a fixed coefficient matrix and $\varphib_t$ is a perturbed term.
	By recursively applying the residual dynamics, it has $\begin{bmatrix}
\xib_{t+1} \\
\xib_{t} 
\end{bmatrix} = \G^{t+1}\begin{bmatrix}
\xib_{0} \\
\xib_{-1} 
\end{bmatrix} + \sum_{s=0}^t \G^{t-s} \begin{bmatrix}
\varphib_s \\ \bm{0}
\end{bmatrix}$.
	Additionally, there is a bound for the multiplication between the powers of $\G$ and any vector $\x$ that $\|\G^i\x\|\leq c \rho^i \|\x\|$, where $0<\rho<1$ and $c>0$ is a constant.
	\item Secondly, we introduce the inductive hypotheses. 
	Assume that (i) the distance between the parameter $\w_i$ and its initial value $\w_0$ has a bound $R > 0$ as $\| \w_i - \w_0\| \leq R$, and (ii) the residual dynamics satisfies $\left\|\begin{bmatrix}
\xib_{i} \\
\xib_{i-1} 
\end{bmatrix}\right\| \leq 2c\theta^i \left\|\begin{bmatrix}
\xib_{0} \\
\xib_{-1} 
\end{bmatrix}\right\|$ for any $i\leq t$, where  $\rho < \theta <1$.
	\item Finally, we prove the convergence of NAG by induction.
	The above two hypotheses trivially hold for the base case $i = 0$. Assume them hold for any $i \leq t$.
	Applying the inductive hypotheses, we can derive the upper bounds for $\|\varphib_t\|$ and the distance $\|\w_{t+1} - \w_0\|$.
	Combining the bound of $\|\varphib_t\|$ and $\|\G^i \x\| \leq c\rho^i\|\x\|$, one can prove $\left\|\sum_{s=0}^t \G^{t-s} \begin{bmatrix}
\varphib_s \\ \bm{0}
\end{bmatrix}\right\| \leq c\theta^{t+1}\left\|\begin{bmatrix}
\xib_{0} \\
\xib_{-1} 
\end{bmatrix}\right\|$.
In the end, it has $\left\|\begin{bmatrix}
\xib_{t+1} \\
\xib_{t} 
\end{bmatrix}\right\| \leq \left\|\G^{t+1}\begin{bmatrix}
\xib_{0} \\
\xib_{-1} 
\end{bmatrix}\right\| + \left\|\sum_{s=0}^t \G^{t-s} \begin{bmatrix}
\varphib_s \\ \bm{0}
\end{bmatrix}\right\|  \leq  2c \theta^{t+1} \left\|\begin{bmatrix}
\xib_{0} \\
\xib_{-1} 
\end{bmatrix}\right\|$.
\end{enumerate}

Denote $\X = (\x_1, \cdots, \x_n) \in \reals^{d_x \times n}$ as the feature matrix and $\Y = (\y_1, \cdots, \y_n) \in \reals^{d_y \times n}$ as the corresponding label matrix.
When $m \geq d_y$, it is noted that the two models we analyzed have the same expressive power as the linear model.
Following the assumption in~\cite{wang2020provable,DH19,HXP20},
we assume there exists a $\W^*$ satisfying $\Y = \W^*\X$, $\X \in \reals^{d_x \times r}$, and $r = rank(\X)$ without losing generality (refer to Appendix B in~\cite{DH19} for details).


\subsection{Deep fully-connected linear neural network}

Denote $\U=\frac{1}{\sqrt{m^{L-1} d_{y}}}\W^{L:1} \X$ as the outputs of the network.
	Denote $\M_t^l$ as the momentum term in the $t$-th iteration for the 
$l$-th layer.
We first introduce the residual dynamics of NAG in training an $L$-layer fully-connected linear neural network.
\begin{lemma} 
\label{deepl: lemma1}
Denote
\[
\textstyle \H_t^{lin} \textstyle = \frac{1}{ m^{L-1} d_y } \sum_{l=1}^L [ (\W^{l-1:1}_t \X)^\top (\W^{l-1:1}_t \X ) 
\otimes
  \W^{L:l+1}_t (\W^{L:l+1}_t)^\top ]   \in \reals^{d_y n \times d_y n}.
\]
Applying NAG  for training a fully-connected linear neural network with L hidden layers, the residual dynamics follows
\begin{equation}
\label{deepl: dynal}
\begin{bmatrix}
\xib_{t+1} \\
\xib_{t} 
\end{bmatrix}
= 
\begin{bmatrix}
(1+\beta)(\I_{d_y n} - \eta \H_0^{lin}) & \beta ( -\I_{d_y n} + \eta \H_{0}^{lin}  )   \\
\I_{d_y n} & \bm{0}_{d_y n} 
\end{bmatrix}
\begin{bmatrix}
\xib_{t} \\
\xib_{t-1} 
\end{bmatrix}
+
\begin{bmatrix}
\varphib_t \\ \bm{0}_{d_y n}
\end{bmatrix},
\end{equation}
where
\[
\begin{aligned}
\xib_t &    = \text{vec}(\U_t - \Y), \;\;
\varphib_t  = \phib_t + \psib_t  + \iotab_t, \\
 \phib_t &  = \frac{1}{\sqrt{m^{L-1} d_y} } \v\left( \Phi_t \X\right) \text{ with }
\Phi_t    = \Pi_l ( \W^{l}_t + \M_t^l )
- \W^{L:1}_t  - \sum_{l=1}^L \W^{L:l+1}_t \M_t^l \W^{l-1:1}_t,\\
\psib_t & =\frac{1}{\sqrt{m^{L-1} d_y} } 
\v\left( (L-1) \beta \W^{L:1}_{t}  \X + \beta  \W^{L:1}_{t-1} \X 
 - \beta \sum_{l=1}^L \W^{L:l+1}_t \W^{l}_{t-1} \W^{l-1:1}_{t} \X \right) \\
 &+ \frac{\eta\beta}{\sqrt{ m^{L-1} d_y}}\v\left( (\sum_{l=1}^L \W^{L:l+1}_{t} \frac{ \partial \ell(\W^{L:1}_{t-1})}{ \partial \W^{l}_{t-1} } \W^{l-1:1}_{t} - \sum_{l=1}^L \W^{L:l+1}_{t-1} \frac{ \partial \ell(\W^{L:1}_{t-1})}{ \partial \W^{l}_{t-1} } \W^{l-1:1}_{t-1})\X\right),\\
  \iotab_t  & = -\eta(1+\beta) (\H_t^{lin} - \H_0^{lin}) \xib_t + \eta\beta(\H_{t-1}^{lin} - \H_0^{lin})\xib_{t-1}.
\end{aligned}
\]
\end{lemma}
Lemma~\ref{deepl: lemma1} shows that the residual errors of two consecutive iterates follow a linear dynamical system with a perturbed term $[\varphib; \0]$.
When the gram  matrix $\H_0^{lin}$ is positive-definite, the spectral norm of the constant coefficient matrix of (\ref{deepl: dynal}) is less than 1 with specific hyperparameters $\eta$ and $\beta$ according to Lemma~\ref{supportlemma1} in~\ref{supporting} 
The details of the proof for verifying the positive-definite of $\H^{lin}_0$ are referred to (\ref{deep1: bound1}). 
If the perturbed term is small enough, we can bound the residual error.

In the following Lemma, we show that the three parts of $\varphib$ can be bounded based on the inductive hypotheses (i) the residual error decrease at a linear rate, and (ii) the parameters of the network are not far from initial values.

\begin{lemma} \label{deepl: lemma2}
Following the settings in Theorem~\ref{thm:LinearNet}, for any $s \leq t$, assume (i) the residual dynamics satisfies
$ \left\|
\begin{bmatrix}
\xib_{s} \\
\xib_{s-1} 
\end{bmatrix}
\right\| \leq 24\sqrt{\kappa}\theta^{s} 
\left\|
 \begin{bmatrix}
\xib_{0} \\
\xib_{-1} 
\end{bmatrix}\right\|$, and (ii) $\forall l \in [L], \forall s \leq t$,
the distance between the parameter $\W^{l}_s$ and its initial has a bound as $\| \W^{l}_s - \W^{l}_0 \|_F \leq R^{lin} = 
\frac{792 \| \X \| \sqrt{d_y\kappa}}{ L \sigma_{\min}^2(\X) } B_0$,
then 
\[
\begin{aligned}
\| \phib_t \| 
\leq \frac{1}{180\sqrt{\kappa}}\theta^{2t}\|\U_0 -\Y\|_F &, \quad
\| \psib_t \| 
\leq  \frac{1}{90\sqrt{\kappa}} \theta^{2t}\| \U_0 -\Y \|_F + \frac{2}{23\sqrt{\kappa}}\theta^t \|\U_0 -\Y\|_F \\
 \| \iotab_t \|  &\leq  \frac{5}{39\sqrt{\kappa}}\theta^t  \|\U_0 -\Y\|_F.
\end{aligned}
\]
Consequently,  $\varphib_t$ in Lemma~\ref{deepl: lemma1} can be bounded by
\[
\| \varphib_t \| \leq
\frac{1}{60\sqrt{\kappa}} \theta^{2t}\| \U_0 -\Y \|_F+
 \frac{5}{23\sqrt{\kappa}} \theta^t\|\U_0 -\Y\|_F .
\]
\end{lemma}

From Lemma~\ref{deepl: lemma2}, it is observed that the bounds of $\phib$, $\psib$ and $\iotab$ all decrease at a linear rate, leading to a controllable perturbed term $\varphib$.
Before introducing the convergence result of NAG, we provide the bound of the distance between $\W_{t+1}^l$ and its initial for any $l \in [m]$.

\begin{lemma}~\label{deepl: lemma_distance}
Following the settings in Theorem~\ref{thm:LinearNet}, for any $s \leq t$, assume the residual dynamics satisfies
$\textstyle \left\|
\begin{bmatrix}
\xib_{s} \\
\xib_{s-1} 
\end{bmatrix}
\right\| \leq \theta^{s} 
24\sqrt{\kappa}
\left\|
 \begin{bmatrix}
\xib_{0} \\
\xib_{-1} 
\end{bmatrix}\right\|,
$ 
then 
\[
\| \W^{l}_{t+1} - \W^{l}_0 \|_F \leq R^{lin} = \frac{792 \| \X \| B_0 \sqrt{d_y\kappa}}{ L \sigma_{\min}^2(\X) } .
\]
\end{lemma}

Finally, with specific hyperparameters, NAG has the following convergence result in training the deep fully-connected linear neural network.
\begin{theorem} \label{thm:LinearNet}
Denote $\lambda_{min} = (0.8)^4 L \sigma^2_{min}(\X) / d_y$, $\lambda_{max} = (1.2)^4 L \sigma^2_{\max}(\X) / d_y$, $\kappa =  \lambda_{max}/\lambda_{min}$, $\theta = 1 - \frac{1}{2\sqrt{\kappa}}$ and $B_0^2 = \O(\max\{1, \frac{\log(r/\delta)}{d_y}, \|\W^*\|^2\}\|\X\|_F^2)$.
By setting $\eta = \frac{1}{2\lambda_{max}}$, $\beta = \frac{3\sqrt{\kappa} - 2}{3\sqrt{\kappa} + 2}$ and $m = \Omega(L\max\{r \kappa^5 d_y(1+\|\W^*\|^2), r\kappa^5\log\frac{r}{\delta}, \log L\})$, for any $t\geq 0$, with probability at least $1-\delta$ over the random Gaussian initialization,
the residual error of NAG in training an $L$-layer fully-connected linear neural network has the following bound for any $t\geq 0$
\begin{equation} \label{eq:thm-LinearNet}
\left\|
\begin{bmatrix}
\xib_t \\
\xib_{t-1} 
\end{bmatrix}
\right\| \leq 24 \sqrt{\kappa} \left( 1 - \frac{1}{2 \sqrt{\kappa} } \right)^{t} 
\left\|
 \begin{bmatrix}
\xib_0 \\
\xib_{-1}
\end{bmatrix}
\right\|.
\end{equation}
\end{theorem}

\renewcommand\arraystretch{1.3}
\begin{table*}
\caption{Summary of the convergence results for the deep fully-connected linear neural network under the random Gaussian initialization. 
Let $m$, $L$ and $d_y$ denote the width, depth and output dimensions of neural network. 
Let $\X$ denotes the training feature matrix. 
Let $\delta$ denotes the failure probability. 
Let $t$ denotes the iteration number.
Define $r = rank(\X)$, $\lambda_{min} = 0.8^4 L \sigma_{min}^2(\X)/d_y$ , $\lambda_{max} = 1.2^4 L \sigma_{max}^2(\X)/d_y$ and $\kappa = \lambda_{max}/\lambda_{min}$. }
\label{table1}
\centering
\begin{tabular}{ | M{1.3cm}| M{4.5cm}| M{3.6cm}| M{2.5cm}|} 
\hline
{\bf Method} & Width $m$ & \small{Hyperparameters} & \small{Convergence rate} \\ \hline
GD~\cite{DH19} & \small{$\Omega(L\max\{r \kappa^3 d_y(1+\|\W^*\|^2)$, $r\kappa^3\log{r}/{\delta}, \log L\})$} &  $\eta =\O( \frac{d_y}{L\|\X^{\top}\X\|}) $ & $( 1 - \O(\frac{1}{{\kappa}}))^t$ \\ [1.4ex] \hline
{NAG} & \small{$\Omega(L\max\{r \kappa^5 d_y(1+\|\W^*\|^2)$, $r\kappa^5\log {r}/{\delta}, \log L\})$} & $\eta = \frac{1}{2 \lambda_{max}},\beta = \frac{3\sqrt{\kappa} - 2}{3\sqrt{\kappa} + 2} $ & $( 1 - \frac{1}{2 \sqrt{\kappa}})^t$ \\ [1.4ex] \hline
\end{tabular}
\end{table*}
\noindent\textbf{Remarks}.
The above theorem shows that NAG can reach the global minimum at a $( 1 - \frac{1}{2 \sqrt{\kappa}})^t$ rate.
Compared to the $( 1 - \O(\frac{1}{{\kappa}}))^t$ rate of GD~\cite{DH19}, our results demonstrate that NAG converges faster than GD.
The details of the over-parameterization and hyperparameters selection can be found in Table~\ref{table1}.
Note that the width $m$ linearly depends on the depth $L$, which means deeper fully-connected linear neural network needs more wider layers to obtain convergence for NAG.

\subsection{Deep linear ResNet}
In this subsection, we provide the convergence result of NAG in training an $L$-layer deep linear ResNet under the zero initialization.
For brevity, we define $\tilde{\W}^{l} = \I + \W^{l}$ and $\tilde{\W}^{j:i} = \Pi_{l=i}^j(\I+\W^{l})$.

\begin{lemma}
\label{deepres: lemm1}
Denote
\[
\textstyle \H_t^{res} \textstyle =  \sum_{l=1}^L \left[ \left( (\tilde{\W}^{l-1:1}_t \A\X)^\top (\tilde{\W}^{l-1:1}_t \A\X)
 \right)  \otimes \left( \B\tilde{\W}^{L:l+1}_t (\B\tilde{\W}^{L:l+1}_t)^\top \right)    \right ]   \in \reals^{d_y n \times d_y n}.
\]
Applying NAG for training an $L$-layer linear ResNet, the residual dynamics satisfies
\begin{equation}
\begin{bmatrix}
\xib_{t+1} \\
\xib_{t} 
\end{bmatrix}
 = 
\begin{bmatrix}
(1+\beta)(\I_{d_y n} - \eta \H_0^{res}) & \beta ( -\I_{d_y n} + \eta \H_{0}^{res}  )   \\
\I_{d_y n} & \0_{d_y n} 
\end{bmatrix}
\begin{bmatrix}
\xib_{t} \\
\xib_{t-1} 
\end{bmatrix}
+
\begin{bmatrix}
\varphib_t \\ \0_{d_y n}
\end{bmatrix}
\end{equation}
where
\[
\begin{aligned}
\xib_t &= \text{vec}(\U_t -\Y) \in \reals^{d_y n}, 
\varphib_t = \phib_t + \psib_t + \iotab_t \\
 \phib_t &  = \v( \B\Phi_t \A\X) \text{ with }
\Phi_t    = \Pi_l ( \tilde{\W}^{l}_t + \M_t^l )
- \tilde{\W}^{L:1}_t  - \sum_{l=1}^L \tilde{\W}^{L:l+1}_t \M_t^l \tilde{\W}^{l-1:1}_t, 
\\   \psib_t & = \v\left(  \B ( (L-1) \beta \tilde{\W}^{L:1}_{t} + \beta  \tilde{\W}^{L:1}_{t-1}
- \beta \sum_{l=1}^L \tilde{\W}^{L:l+1}_t \tilde{\W}^{l}_{t-1} \tilde{\W}^{l-1:1}_{t} ) \A\X \right) \\
&+ \v\left(\eta\beta \B (\sum_{l=1}^L \tilde{\W}^{L:l+1}_{t} \frac{ \partial \ell(\W^{L:1}_{t-1}) }{ \partial \W^{l}_{t-1} } \tilde{\W}^{l-1:1}_{t} - \sum_{l=1}^L \tilde{\W}^{L:l+1}_{t-1} \frac{ \partial \ell(\W^{L:1}_{t-1})}{ \partial \W^{l}_{t-1} } \tilde{\W}^{l-1:1}_{t-1}) \A\X\right),
\\ \iotab_t  & = -\eta(1+\beta) (\H_t^{res} - \H_0^{res}) \xib_t + \eta\beta(\H_{t-1}^{res} - \H_0^{res})\xib_{t-1}.
\end{aligned}
\]
\end{lemma}

\begin{lemma} \label{deepres: lemm2}
Following the settings in Theorem~\ref{thm:resnet}, for any $s \leq t$, assume (a) the residual dynamics satisfies
$ \left\|
\begin{bmatrix}
\xib_{s} \\
\xib_{s-1} 
\end{bmatrix}
\right\| \leq 24\sqrt{\kappa} \theta^{s}  
\left\|
 \begin{bmatrix}
\xib_{0} \\
\xib_{-1} 
\end{bmatrix}
\right\|,$
and (b) for all $l \in [L]$ and for any $s \leq t$,
$\| \W^{l}_s \|_F \leq R^{res} = 
1/(2000L\kappa)$,
then 
\begin{eqnarray}
\| \phib_t \| 
\leq 
\frac{1}{ 180\sqrt{\kappa}}   
  \theta^{2t}    \| \U_0 -\Y \|_F   , \! 
 \| \psib_t \| 
\leq 
\frac{2}{ 45\sqrt{\kappa}}   
  \theta^{2t}    \| \U_0 -\Y \|_F  , 
 \| \iotab_t \|  \leq \frac{3}{ 17\sqrt{\kappa}}   
  \theta^{t}    \| \U_0 -\Y \|_F .\nonumber
\end{eqnarray}
Consequently,  $\varphib_t$ in Lemma~\ref{deepres: lemm1} satisfies
\[
\| \varphib_t \| \leq
\frac{1}{ 30\sqrt{\kappa}}   
  \theta^{2t}    \| \U_0 -\Y \|_F  +\frac{3}{ 17\sqrt{\kappa}}   
  \theta^{t}    \| \U_0 -\Y \|_F  .
\]
\end{lemma}

\begin{theorem} \label{thm:resnet}
Denote $\lambda_{min} = (0.9)^4 L \alpha^2\gamma^2 m^2 \sigma^2_{min}(\X) $, $\lambda_{max} = (1.1)^4 L \alpha^2\gamma^2 m^2 \sigma^2_{\max}(\X)$, $\kappa =  \lambda_{max}/\lambda_{min}$ and $\theta = 1 - \frac{1}{2\sqrt{\kappa}}$.
By setting $\eta = \frac{1}{2L\|\A\|^2\|\B\|^2\|\X\|^2}$, $\beta = \frac{3\sqrt{\kappa} -2}{3\sqrt{\kappa} + 2}$ and $m=\Omega(\max\{d_y r \kappa^5 \log(n/\delta), \sqrt{r}\kappa^{2.5}a\|\W^*\|/\alpha\gamma, d_x + d_y + \log(1/\delta)\})$, with probability at least $1-\delta$ over the random initialization,
the residual error of NAG in training an $L$-layer linear ResNet has the following bound for any $t\geq 0$,
\begin{equation}
\left\|
\begin{bmatrix}
\xib_t \\
\xib_{t-1} 
\end{bmatrix}
\right\| \leq 24 \sqrt{\kappa}\left( 1 - \frac{1}{2 \sqrt{\kappa} } \right)^{t}
\left\|
 \begin{bmatrix}
\xib_0 \\
\xib_{-1}
\end{bmatrix}
\right\|.
\end{equation} 
\end{theorem}
\begin{table*}
\caption{Summary of the convergence results for the deep linear ResNet under the zero initialization. 
Let $m$ and $L$ denote the width and the depth of the neural network. 
Let $d_x$ and $d_y$ denote the input and output dimensions of the neural network.
Let $\X$ denotes the training feature matrix.
Let $n$ denotes the number of the training instances. 
Let $\delta$ denotes the failure probability. 
Let $t$ denotes the iteration number.
Define $a = \|\A\|\|\B\|\|\X\|$, $ r = rank(\X)$, $\lambda_{min} = (0.9)^4 L \alpha^2\gamma^2 m^2 \sigma^2_{min}(\X) $, $\lambda_{max} = (1.1)^4 L \alpha^2\gamma^2 m^2 \sigma^2_{\max}(\X)$ and $\kappa = \lambda_{max}/\lambda_{min}$. }
\label{table2}
\centering
\begin{tabular}{ | M{1.3cm}| M{5.4cm}| M{3.3cm}| M{2.4cm}|} 
\hline
{\bf Method} & Width $m$ & \small{Hyperparameters} & \small{Convergence rate} \\ \hline
GD~\cite{DBLP:conf/iclr/ZouLG20} & \small{$\Omega(\max\{d_y r \kappa^2 \log(n/\delta)$, $\sqrt{r}\kappa\|\W^*\|/\alpha\gamma,d_x + d_y + \log(1/\delta)\})$} &  \small{$\eta =\O( \frac{1}{La(B_0 +a)}) $} & $( 1 - \O(\frac{1}{{\kappa}}))^t$  \\ [1.4ex] \hline
{NAG} & \small{$\Omega(\max\{d_y r \kappa^5 \log(n/\delta)$, $\sqrt{r}\kappa^{2.5}a\|\W^*\|/\alpha\gamma,d_x + d_y + \log(1/\delta)\})$} & \small{$\eta = \frac{1}{2L a^2}$
$\beta = \frac{3\sqrt{\kappa} - 2}{3\sqrt{\kappa} + 2} $} & $( 1 - \frac{1}{2 \sqrt{\kappa}})^t$ \\ [1.4ex] \hline
\end{tabular}
\end{table*}
\noindent\textbf{Remarks}. NAG is capable of attaining the global minimum for training an $L$-layer linear ResNet, where the convergence rate is faster than that of GD~\cite{DBLP:conf/iclr/ZouLG20}.
Moreover, note that the requirement of the width  for deep linear ResNets has no dependence on $L$.

\section{Conclusion and future work}
In this paper, we analyze the convergence of NAG in training deep linear networks, including deep fully-connected neural networks and deep linear ResNets.
We show that NAG is capable of converging to the global minimum for above two architectures of neural networks.
Moreover, the convergence results demonstrate NAG achieves acceleration over GD.

In future work, we will consider other types of neural networks, such as deep fully-connected neural networks with non-linear activation functions, deep convolution neural networks and so on.
In addition, we will extend our analysis to  other modern momentum methods, such as Adam~\cite{KB15}, AMSGrad~\cite{RKK18} and AdaBound~\cite{LXLS19}.
We hope our results may provide insights to understand the training behavior of momentum methods for neural networks.

\bibliographystyle{elsarticle-num}
\bibliography{NAG_deep}

\appendix

\onecolumn

\section{Deep fully-connected linear neural network
 }

\subsection{Proof of Lemma~\ref{deepl: lemma1}}
\begin{proof} 
According to the update rule (\ref{procudure: NAG_2}) of NAG,
it has
\begin{equation} \label{deepl: 1}
\W^{L:1}_{t+1} = \Pi_l \left( \W^{l}_t + \M_t^l \right)
= \W^{L:1}_t + \sum_{l=1}^L \W^{L:l+1}_t \M_t^l \W^{l-1:1} + \Phi_t,
\end{equation}
where $ \M_t^l = \beta ( \W^{l}_t - \W^{l}_{t-1} ) - \eta\beta(\frac{ \partial \ell(\W^{L:1}_t)}{ \partial \W^{l}_t } - \frac{ \partial \ell(\W^{L:1}_{t-1})}{ \partial \W^{l}_{t-1} })  - \eta \frac{ \partial \ell(\W^{L:1}_t)}{ \partial \W^{l}_t }
$ denotes the momentum term for the layer $l$ at iteration $t$,
and 
$\Phi_t$ contains all the high-order terms of the momentum term.
Then (\ref{deepl: 1}) can be rewritten as
\begin{eqnarray}
\W^{L:1}_{t+1} 
\!\!\!\!& =& \!\!\!\!\W^{L:1}_t - (1+\beta)\eta \sum_{l=1}^L \W^{L:l+1}_t \frac{ \partial \ell(\W^{L:1}_t)}{ \partial \W^{l}_t } \W^{l-1:1}_t +  \eta \beta \sum_{l=1}^L \W^{L:l+1}_t \frac{ \partial \ell(\W^{L:1}_{t-1})}{ \partial \W^{l}_{t-1} } \W^{l-1:1}_t \nonumber \\
&+& \sum_{l=1}^L \W^{L:l+1}_t \beta ( \W^{l}_t - \W^{l}_{t-1} ) \W^{l-1:1}_t \nonumber\\ 
& = & \W^{L:1}_t - \eta(1+\beta) \sum_{l=1}^L \W^{L:l+1}_t \frac{ \partial \ell(\W^{L:1}_t)}{ \partial \W^{l}_t } \W^{l-1:1}_t + \beta ( \W^{L:1}_t - \W^{L:1}_{t-1} ) \nonumber \\
\!\!\!\!&+& \!\!\!\eta\beta \sum_{l=1}^L \W^{L:l+1}_{t-1} \frac{ \partial \ell(\W^{L:1}_{t-1})}{ \partial \W^{l}_{t-1} } \W^{l-1:1}_{t-1} + (L-1) \beta \W^{L:1}_{t} + \beta  \W^{L:1}_{t-1}
- \beta \sum_{l=1}^L \W^{L:l+1}_t \W^{l}_{t-1} \W^{l-1:1}_{t} \nonumber\\ 
&+& \eta\beta(\sum_{l=1}^L \W^{L:l+1}_{t} \frac{ \partial \ell(\W^{L:1}_{t-1})}{ \partial \W^{l}_{t-1} } \W^{l-1:1}_{t} - \sum_{l=1}^L \W^{L:l+1}_{t-1} \frac{ \partial \ell(\W^{L:1}_{t-1})}{ \partial \W^{l}_{t-1} } \W^{l-1:1}_{t-1}) + \Phi_t. \nonumber 
\end{eqnarray}
Multiplying both sides of the above equality with $\frac{1}{ \sqrt{ m^{L-1} d_y} } \X$, it has
\begin{eqnarray}
\U_{t+1} \!\!\!\!& = &\!\!\!\! \U_t - \frac{\eta(1+\beta)}{ m^{L-1} d_y } \sum_{l=1}^L \W^{L:l+1}_t  (\W^{L:l+1}_t)^\top
( \U_t -\Y )   (\W^{l-1:1}_t \X)^\top  \W^{l-1:1}_t \X  + \beta  (\U_t - \U_{t-1}) \nonumber \\ 
& + &  \frac{\eta\beta}{ m^{L-1} d_y } \sum_{l=1}^L \W^{L:l+1}_{t-1}  (\W^{L:l+1}_{t-1})^\top
( \U_{t-1} -\Y )   (\W^{l-1:1}_{t-1} \X)^\top  \W^{l-1:1}_{t-1} \X  \nonumber\\
& +&\!\!\!\! \frac{1}{ \sqrt{ m^{L-1} d_y} } \left( (L-1) \beta \W^{L:1}_{t} + \beta  \W^{L:1}_{t-1}
- \beta \sum_{l=1}^L \W^{L:l+1}_t \W^{l}_{t-1} \W^{l-1:1}_{t} \right) \X
+
\frac{1}{ \sqrt{ m^{L-1} d_y} } \Phi_t \X \nonumber\\
& + &\frac{\eta\beta}{\sqrt{ m^{L-1} d_y}} (\sum_{l=1}^L \W^{L:l+1}_{t} \frac{ \partial \ell(\W^{L:1}_{t-1})}{ \partial \W^{l}_{t-1} } \W^{l-1:1}_{t} - \sum_{l=1}^L \W^{L:l+1}_{t-1} \frac{ \partial \ell(\W^{L:1}_{t-1})}{ \partial \W^{l}_{t-1} } \W^{l-1:1}_{t-1})\X. \nonumber
\end{eqnarray}
Using $\text{vec}(\A\bm{C}\B) = (\B^\top \otimes \A) \text{vec}(\bm{C})$, it has 
\begin{eqnarray} \label{deepl: 2}
&&\v(\U_{t+1}) - \v(\U_t)\nonumber \\
&  =& \!\!\!\!  - \eta(1+\beta) \H_t^{lin} \v( \U_t -\Y ) +
\beta  \left( \v(\U_{t}) - \v(\U_{t-1})  \right) + \eta\beta \H_{t-1}^{lin}\v(\U_{t-1} -\Y)
\nonumber\\
 & + &
\v\left(  \frac{1}{ \sqrt{ m^{L-1} d_y} } ( (L-1) \beta \W^{L:1}_{t} + \beta  \W^{L:1}_{t-1}
- \beta \sum_{l=1}^L \W^{L:l+1}_t \W^{l}_{t-1} \W^{l-1:1}_{t} ) \X \right) \nonumber\\
 &  + & \!\!\!\! \v\left(\frac{\eta\beta}{\sqrt{ m^{L-1} d_y}} (\sum_{l=1}^L \W^{L:l+1}_{t} \frac{ \partial \ell(\W^{L:1}_{t-1})}{ \partial \W^{l}_{t-1} } \W^{l-1:1}_{t} - \sum_{l=1}^L \W^{L:l+1}_{t-1} \frac{ \partial \ell(\W^{L:1}_{t-1})}{ \partial \W^{l}_{t-1} } \W^{l-1:1}_{t-1})\X\right)\nonumber\\
 & +&
\frac{1}{ \sqrt{ m^{L-1} d_y} } \v( \Phi_t \X),
\end{eqnarray}
where 
\begin{equation}
\H_t^{lin} = \frac{1}{  m^{L-1} d_y } \sum_{l=1}^L \left[ \left( (\W^{l-1:1}_t \X)^\top (\W^{l-1:1}_t \X)
 \right)  \otimes \W^{L:l+1}_t (\W^{L:l+1}_t)^\top    \right].
\end{equation} 

Then (\ref{deepl: 2}) can be reformulated as
\begin{equation} \label{eq:L2}
\begin{split}
\begin{bmatrix}
\xib_{t+1} \\
\xib_{t} 
\end{bmatrix}
& = 
\begin{bmatrix}
(1+\beta)(\I_{d_y n} - \eta \H_t^{lin}) & \beta ( -\I_{d_y n} + \eta \H_{t-1}^{lin}  ) \\
\I_{d_y n} & \0_{d_y n} 
\end{bmatrix}
\begin{bmatrix}
\xib_{t} \\
\xib_{t-1} 
\end{bmatrix}
+
\begin{bmatrix}
\phib_t + \psib_t \\ \0_{d_y n}
\end{bmatrix}
\\ & = 
\begin{bmatrix}
(1+\beta)(\I_{d_y n} - \eta \H_0^{lin}) & \beta ( -\I_{d_y n} + \eta \H_{0}^{lin}  )   \\
\I_{d_y n} & \0_{d_y n} 
\end{bmatrix}
\begin{bmatrix}
\xib_{t} \\
\xib_{t-1} 
\end{bmatrix}
+
\begin{bmatrix}
\varphib_t \\ \0_{d_y n}
\end{bmatrix}
,
\end{split}
\end{equation}
where $\varphib_t = \phib_t + \psib_t + \iotab_t \in \reals^{d_y n}$ and 
 $\I_{d_y n}$ is the $d_y n \times d_y n$-dimensional identity matrix.
\end{proof}

Before presenting the proof of our main results, we introduce some supporting lemmas.

\begin{lemma}{(Proposition 6.2 and Proposition 6.3 in \cite{DH19})}\label{deepl: lemma4}
For any $i \in (1, L]$ and $j \in [1, L)$, with probability at least $1 - exp(-\Omega(m/L))$, it has
\begin{eqnarray}
	\sigma_{max}(\W^{L:i}_0) &\leq& 1.2m^{\frac{L-i+1}{2}}, \;\; \sigma_{min}(\W^{L:i}_0) \geq 0.8m^{\frac{L-i+1}{2}} \nonumber\\
	\sigma_{max}(\W^{j:1}_0\X) &\leq& 1.2m^{\frac{j}{2}}\sigma_{max}(\X), \;\; \sigma_{min}(\W^{j:1}_0 \X) \geq 0.8m^{\frac{j}{2}}\sigma_{min}(\X) \nonumber
\end{eqnarray}
\end{lemma}

\begin{lemma}{(Proposition 6.4 in \cite{DH19})}
\label{deepl: lemma_5}
For any $1 < i\leq j < L$,with probability at least $1 - exp(-\Omega(m/L))$, it has
\begin{eqnarray}
\|\W^{j:i}_0\| \leq \O(\sqrt{L}m^{\frac{j-i+1}{2}})
\end{eqnarray}
\end{lemma}

\begin{lemma}{(Proposition 6.5 in \cite{DH19})}\label{deepl: lemma2.5}
Suppose $m \geq CL\log L$ for a sufficiently large constant $C > 0$.
With probability at least $1- exp(-\Omega(m/L)) -\delta/2 $, it has
\begin{equation}
	\ell(0)\leq B_0^2 = \O(\max\{1, \frac{\log(r/\delta)}{d_y}, \|\W^*\|^2\}\|\X\|_F^2)
\end{equation}
\end{lemma}

Therefore, according to the properties of the Kronecker product, it has the bounds for the spectrum of the matrix $\H_0$ as
\begin{eqnarray}\label{deep1: bound1}
	\lambda_{min}(\H_0^{lin}) \geq \frac{0.8^4L\sigma^2_{min}(\X)}{d_y}, \;\; \lambda_{max}(\H_0^{lin}) \leq \frac{1.2^4 L \sigma^2_{max}(\X)}{d_y}
\end{eqnarray}
For abuse of the notation, we define $\lambda_{min} = \frac{0.8^4L\sigma^2_{min}(\X)}{d_y}$, $\lambda_{max} = \frac{1.2^4 L \sigma^2_{max}(\X)}{d_y}$ and $\kappa = \lambda_{max}/\lambda_{min}$.

\begin{lemma}{(Claim 7.2 in~\cite{DH19})}\label{deep1: lemm4.5}
Suppose $m = \Omega(L\max\{r \kappa^5 d_y(1+\|\W^*\|^2), r\kappa^5\log\frac{r}{\delta}, \log L\})$.
 Assume $\|\W^{k}_t - \W^{k}_0\|_F \leq R^{lin}= \frac{792 \| \X \| B_0 \sqrt{d_y\kappa}}{ L \sigma_{\min}^2(\X) } $
 for any $k \in [L]$ and any $t$, it has
\begin{eqnarray}
\label{deepl: bound_param_t}
	\sigma_{max}(\W^{L:i}_t) \leq 1.25 m^{\frac{L-i+1}{2}} &,& \;\; \sigma_{min}(\W^{L:i}_t) \geq 0.75 m^{\frac{L-i+1}{2}} \;\; \forall 1< i\leq L \nonumber\\
	\sigma_{max}(\W^{j:1}_t\X) \leq 1.25 m^{\frac{j}{2}}\sigma_{max}(\X) &,& \;\; \sigma_{min}(\W^{j:1}_t \X) \geq 0.75 m^{\frac{j}{2}}\sigma_{min}(\X) \;\; \forall 1 \leq j < L \nonumber \\
	\|\W^{j:i}_t\| &\leq& \O(\sqrt{L}m^{\frac{j-i+1}{2}}) \;\; \forall 1 <i\leq j < L
\end{eqnarray}
\end{lemma}

\begin{proof}
For completeness, we replicate the proof in~\cite{DH19}, but consider a universe $R^{lin}$.

For the bounds of  $\|\W^{L:i}_t - \W^{L:i}_0\|$ with any $1<i\leq L$, it has
\begin{eqnarray} \label{deep1: lemm4.51}
\| \W^{L:i}_t - \W^{L:i}_0 \|_F & \leq& 
1.2\sum_{l=1}^{L-i+1} { L-i+1 \choose l } (R^{lin})^l (\O(\sqrt{L}))^{l}m^{\frac{L-i+1-l}{2}}
\nonumber\\ & \leq &1.2(\sqrt{m})^{L-i+1} \left(  (1+\O(R^{lin}\sqrt{L})/\sqrt{m})^{L-i+1} -1  \right)\nonumber\\
&\leq& 1.2(\sqrt{m})^{L-i+1} \left(  (1+\O(R^{lin}\sqrt{L})/\sqrt{m})^{L} -1  \right)
\nonumber\\ &  
\overset{(a)}{\leq}& 1.2\left(  (1+\frac{1}{C_1L\kappa})^{L} -1  \right)( \sqrt{m} )^{L-i+1} \nonumber\\ 
& \overset{(b)}{\leq} & \!\!\!\!\!\!1.2\left( \!1 \!+\! (e\!-\!1)\frac{1}{C_1\kappa} \!-\! 1\right)\!( \sqrt{m} )^{L-i+1}  \overset{(c)}{\leq} \!\!
 \frac{1}{1500 \kappa } ( \sqrt{m} )^{L-i+1} ,
\end{eqnarray}
where (a) 
uses $m \geq CL^3(R^{lin})^2\kappa^2$, which leads to $O(R^{lin}\sqrt{L})/\sqrt{m} \leq 1/(C_1 L \kappa)$, for some sufficiently large $C>0$ and $C_1 > 0$
(b) uses $(1+x/n)^n\leq e^x, \forall x \geq 0, n >0$ and  Bernoulli's inequality $e^r\leq 1 + (e-1)r,\forall 0 \leq r\leq 1$, and (c) uses any sufficiently larger $C'$. 
Combining Lemma~\ref{deepl: lemma4}, it proves the first line of $\ref{deepl: bound_param_t}$, where the proof of the second line result follows a similar approach.

For any $\|\W^{j:i}_{t}- \W^{j:i}_{0}\|$ with $L > j \geq i> 1$, denotes $\Delta^{k}_t = \W^{k}_t - \W^{k}_0$, it has
\begin{equation}
	\W^{j:i}_t = \left( \W^{j}_0  + \Delta^{j}_t \right)
 \cdots \left( \W^{i}_0 + \Delta^{i}_t  \right).
\end{equation}
Applying binomial theorem, $\|\Delta^{k}_t\| \leq \|\W^{k}_t - \W^{k}_0\|_F \leq R$ and Lemma~\ref{deepl: lemma_5}, it has
\begin{equation} \label{deep1: lemm4.53}
\begin{split}
\| \W^{j:i}_t - \W^{j:i}_0 \|_F & \leq 
\sum_{l=1}^{j-i+1} { j-i+1 \choose l } (R^{lin})^l (\O(\sqrt{L}))^{l+1}m^{\frac{j-i+1-l}{2}}
\\ & \leq \O(\sqrt{L})(\sqrt{m})^{j-i+1} \left(  (1+\O(R^{lin}\sqrt{L})/\sqrt{m})^{j-i+1} -1  \right)\\
&\leq \O(\sqrt{L})(\sqrt{m})^{j-i+1} \left(  (1+\O(R^{lin}\sqrt{L})/\sqrt{m})^{L} -1  \right)
\\ & 
\overset{(a)}{\leq} \O(\sqrt{L}m^{\frac{j-i+1}{2}} ) 
\end{split}
\end{equation}
where (a) 
uses $m > C L^3(R^{lin})^2$ for a sufficiently large $C$.
Moreover, it is noted that $m \geq C_1 L \kappa^5 r \max\{d_y(1+\|\W^*\|^2), \log(r/\delta)\} \geq C_2L^3(R^{lin})^2\kappa^2$ using $\|X\|_F \leq \sqrt{r}\|X\|$ for some sufficient large constants $C_1$ and $C_2$.
Combining the bound of $m$ in Lemma~\ref{deepl: lemma_5}, we complete the proof.
\end{proof}

\subsection{Proof of Lemma~\ref{deepl: lemma2}}
\begin{proof}
According to Lemma~\ref{deepl: lemma1},
$\varphib_t = \phib_t + \psib_t  + \iotab_t \in \reals^{d_y n}$,
where
\begin{eqnarray}
\label{deep1: lemma5_1}
 \phib_t \!\!=\!\! \frac{1}{\sqrt{m^{L-1} d_y}} \v( \Phi_t \X)
\text{ , with } 
\Phi_t   = \Pi_l ( \W^{l}_t + \M_t^l )
\!-\! \W^{L:1}_t \!-\! \sum_{l=1}^L \W^{L:l+1}_t \M_t^l \W^{l-1:1}_t,\nonumber\\
\end{eqnarray}
and
\begin{eqnarray}
 \psib_t \!\!\!& = &\!\!\!\frac{1}{\sqrt{m^{L-1} d_y} } 
\v\left( (L-1) \beta \W^{L:1}_{t}  \X + \beta  \W^{L:1}_{t-1} \X 
 - \beta \sum_{l=1}^L \W^{L:l+1}_t \W^{l}_{t-1} \W^{l-1:1}_{t} \X \right) \nonumber\\&+&\!\!\! \frac{\eta\beta}{\sqrt{ m^{L-1} d_y}}\v\left( (\sum_{l=1}^L \W^{L:l+1}_{t} \frac{ \partial \ell(\W^{L:1}_{t-1})}{ \partial \W^{l}_{t-1} } \W^{l-1:1}_{t} - \sum_{l=1}^L \W^{L:l+1}_{t-1} \frac{ \partial \ell(\W^{L:1}_{t-1})}{ \partial \W^{l}_{t-1} } \W^{l-1:1}_{t-1})\X\right).\nonumber\\
\end{eqnarray}
and
\begin{eqnarray}
& \iotab_t =  -\eta(1+\beta) (\H_t^{lin} - \H_0^{lin}) \xib_t + \eta\beta(\H_{t-1}^{lin} - \H_0^{lin})\xib_{t-1}.
\end{eqnarray}

Applying the subadditivity $\| \varphib_t \| \leq \| \phib_t \| + \|\psib_t \| + \| \iotab_t \|$, we can bound $\|\varphib_t\|$ by separately deriving the bounds of $\| \phib_t \|$, $\|\psib_t \|$ and $\| \iotab_t \|$.

Firstly, we consider the upper bound of the momentum term $\|\M_{t, l}\|$, where $\M_{t, l}$ is composed of the gradients $\frac{ \partial \ell(\W^{L:1}_st)}{ \partial \W^{l}_t}$ and $\frac{ \partial \ell(\W^{L:1}_{t-1})}{ \partial \W^{t-1}_s }$ as shown in (\ref{procudure: NAG_2}).
For any $\frac{ \partial \ell(\W^{L:1}_s)}{ \partial \W^{l}_s } $ with $s\leq t$, it has the bound as
\begin{equation} \label{deepl: lemma5_2}
\begin{split}
 \| \frac{ \partial \ell(\W^{L:1}_s)}{ \partial \W^{l}_s } \|_F
 &
\leq \frac{1}{\sqrt{m^{L-1} d_y} } \| \W^{L:l+1}_s \| \| \U_s -\Y \|_F \| \W^{l-1:1}_s \X\| 
\\ &
\leq \frac{1}{\sqrt{m^{L-1} d_y} } 1.25 m^{\frac{L-l}{2}}  
\theta^s 24\sqrt{\kappa} \sqrt{2} \| \U_0 -\Y \|_F
 1.25 m^{\frac{l-1}{2}} \| \X \|
\\ &
\leq
\frac{54 \| \X \| \sqrt{\kappa}}{\sqrt{d_y}} \theta^s  \| \U_0 -\Y \|_F,
\end{split}
\end{equation}
where the second inequality uses Lemma~\ref{deep1: lemm4.5}, the induction hypothesis 
$\left\| \begin{bmatrix} \xib_s \\ \xib_{s-1} \end{bmatrix} \right\| \leq 24\sqrt{\kappa}\theta^s  \left\| \begin{bmatrix} \xib_0 \\ \xib_{-1} \end{bmatrix} \right\|$ and $\|\xib_{-1}\| = \|\xib_{0}\|$.
Then $\M_t^l$ can be bounded by
\begin{equation}\label{deepl: lemma5_3}
	\begin{split}
	\|\M_t^l\| &= 
\|-\eta\sum_{s=0}^t \beta^{t-s}\left( \frac{ \partial \ell(\W^{L:1}_s)}{ \partial \W^{l}_s } + \beta(\frac{ \partial \ell(\W^{L:1}_s)}{ \partial \W^{l}_s } - \frac{ \partial \ell(\W^{L:1}_{s-1})}{ \partial \W^{l}_{s-1} })\right) \| \\
&\leq \eta(1+\beta)\sum_{s=0}^t\|\beta^{t-s} \frac{ \partial \ell(\W^{L:1}_s)}{ \partial \W^{l}_s }\| + \eta\beta \sum_{s=0}^t\|\beta^{t-s} \frac{ \partial \ell(\W^{L:1}_{s-1})}{ \partial \W^{l}_{s-1} }\| \\
&\leq \frac{54 \| \X \| \sqrt{\kappa}}{\sqrt{d_y}} \eta  \| \U_0 -\Y \|_F \left((1+\beta)\sum_{s=0}^t \beta^{t-s} \theta^s  + \beta \sum_{s=0}^t \beta^{t-s} \theta^{s-1} \right) \\
&\leq  
\frac{168 \| \X \|\sqrt{\kappa}}{\sqrt{d_y}}  \frac{ \theta^{t} }{1 -\theta}     \| \U_0 -\Y \|_F,
 	\end{split}
\end{equation}
where the last inequality uses $\beta \leq \theta^2$.

From the definition of  $\Phi_t$ in (\ref{deep1: lemma5_1}), it includes the summation of all the high-order momentum terms, e.g. 
$\frac{1}{ \sqrt{ m^{L-1} d_y} }\beta \W^{L:k_j+1}_{t} \cdot  \M_t^{k_j} \W^{k_j-1:k_{j-1}+1}_{t} 
\cdot \M_t^{k_{j-1}} \cdots   \M_t^{k_{1}} \cdot \W^{k_1-1:1}_{t} $, where $1 \leq k_1 < \cdots < k_j \leq L$ for any $j \geq 2$.

Thus we can derive the upper bound of $\| \frac{1}{\sqrt{m^{L-1} d_y} } \Phi_t \X \|_F$ as
\begin{eqnarray}\label{deep1: lemma5_4}  
&& \| \frac{1}{\sqrt{m^{L-1} d_y} } \Phi_t \X \|_F
\nonumber\\ 
&\overset{(a)}{\leq}& \frac{1}{ \sqrt{ m^{L-1} d_y}  } \sum_{j=2}^L {L \choose j} 
\left( \eta 
\frac{168 \| \X \|\sqrt{\kappa}}{\sqrt{d_y}}  \frac{ \theta^{t} }{1 -\theta}      \| \U_0 -\Y \|_F  \right)^j (1.25)^2( \O(\sqrt{L}))^{j-1}  m^{\frac{L-j}{2}} \| \X \|
\nonumber\\ &
\overset{(b)}{\leq} & \frac{1.25^2}{ \O(\sqrt{L})\sqrt{ m^{L-1} d_y}  } \sum_{j=2}^L L^j 
\left( \eta
\frac{168 \| \X \|\sqrt{\kappa}}{\sqrt{d_y}}  \frac{ \theta^{t} }{1 -\theta}  \| \U_0 -\Y \|_F \O(\sqrt{L})\right)^j  m^{\frac{L-j}{2}} \| \X \|
\nonumber\\ &
\leq &\frac{1.25^2}{\O(\sqrt{L})}  \sqrt{ \frac{ m}{ d_y}  } \| \X \| \sum_{j=2}^L  
\left( \eta
\frac{ 168 L\| \X \|\sqrt{\kappa}}{\sqrt{m d_y}}  \frac{ \theta^{t} }{1 -\theta}   \| \U_0 -\Y \|_F\O(\sqrt{L})\right)^j,   
\end{eqnarray}
where (a) uses (\ref{deepl: lemma5_3}) and Lemma~\ref{deep1: lemm4.5}
for
 and (b) uses that ${L \choose j  } \leq \frac{L^j}{j!} $

With the specific $\eta = \frac{1}{2\lambda_{max}}=\frac{d_y}{2*1.2^4L\sigma^2_{max}(\X)} $ , it is easy to derive the bound of $\eta
\frac{ 168 L\| \X \|\sqrt{\kappa}}{\sqrt{m d_y}}  \frac{ \theta^{t} }{1 -\theta}  \| \U_0 -\Y \|_F \O(\sqrt{L})$ in (\ref{deep1: lemma5_4}) as
\begin{eqnarray}\label{deepl: lemma5bm}
\eta \frac{ 168 L\| \X \|\sqrt{\kappa}}{\sqrt{m d_y}}  \frac{ \theta^{t} }{1 -\theta}  \| \U_0 -\Y \|_F \O(\sqrt{L})
 &
\leq& 48
\sqrt{ \frac{  d_y\kappa}{ m }  } \frac{1}{ \| \X \| }  \frac{ \theta^{t} }{1 -\theta}   \| \U_0 -\Y \|_F \O(\sqrt{L})
\nonumber\\ & \leq& 0.5,  
\end{eqnarray}
where the last inequality uses the lower bound on the width $m$ as $m\geq C \frac{d_yB_0^2\kappa^2L}{\|\X\|^2}$ with a sufficent large constant ${C}$.
As a result, $\phib_t$ has the bound as
\begin{eqnarray} \label{deepl: phi}
\| \phib_t \| &  = &\| \frac{1}{ \sqrt{ m^{L-1} d_y}  } \Phi_t \X \|_F
\nonumber\\
 & \leq& \frac{1.25^2}{ \O(\sqrt{L})}  \sqrt{ \frac{ m}{ d_y}  } \| \X \| 
\left( \eta
\frac{168 L\| \X \|\sqrt{\kappa}}{\sqrt{m d_y}}  \frac{ \theta^{t} }{1 -\theta}  \| \U_0 -\Y \|_F \O(\sqrt{L}) \right)^2  
\sum_{j=2}^{L-2} 
\left(  0.5 \right)^{j-2}   
\nonumber\\ & \leq& \frac{1.25^2}{ \O(\sqrt{L})}  \sqrt{ \frac{ m}{ d_y}  } \| \X \|
\left( \eta
\frac{168 L\| \X \|\sqrt{\kappa}}{\sqrt{m d_y}}  \frac{ \theta^{t} }{1 -\theta}    \| \U_0 -\Y \|_F \O(\sqrt{L}) \right)^2   
\nonumber\\ 
& \leq & \frac{ \O(\sqrt{L})  }{ \| \X \|}\sqrt{\frac{d_y}{m}}
\left( 
 \frac{ \theta^{t} }{1 -\theta} 64\sqrt{\kappa}  \| \U_0 -\Y \|_F  \right)^2
\overset{}{\leq}\frac{1}{180\sqrt{\kappa}}\theta^{2t}\|\U_0 -\Y\|_F,
\end{eqnarray}
where the last inequality uses $m \geq C\frac{Ld_yB_0^2\kappa^5}{\|\X\|^2}$ for a sufficently large $C \geq 0$.

Then we turn to analyze $\|\psib_t\|$, which is composed of two parts:
$\frac{1}{ \sqrt{ m^{L-1} d_y} }\beta (L-1)  \W^{L:1}_{t} \X + \frac{1}{ \sqrt{ m^{L-1} d_y} }\beta  \W^{L:1}_{t-1} \X
- \frac{1}{ \sqrt{ m^{L-1} d_y} } \beta \sum_{l=1}^L \W^{L:l+1}_t \W^{l}_{t-1} \W^{l-1:1}_{t} \X$ 

\noindent and $ \frac{\eta\beta}{\sqrt{ m^{L-1} d_y}} (\sum_{l=1}^L \W^{L:l+1}_{t} \frac{ \partial \ell(\W^{L:1}_{t-1})}{ \partial \W^{l}_{t-1} } \W^{l-1:1}_{t} - \sum_{l=1}^L \W^{L:l+1}_{t-1} \frac{ \partial \ell(\W^{L:1}_{t-1})}{ \partial \W^{l}_{t-1} } \W^{l-1:1}_{t-1})\X$.

The first part can be rewritten as
\begin{equation} \label{deep1: lemma5psi1}
\begin{aligned}
& \underbrace{\frac{1}{ \sqrt{ m^{L-1} d_y} }   \beta (L-1)  \cdot \Pi_{l=1}^L \left( \W^{l}_{t-1} + \M_{t-1}^{l} \right) \X }_{\text{first term} } + \underbrace{ \frac{1}{ \sqrt{ m^{L-1} d_y} }\beta  \W^{L:1}_{t-1} \X }_{\text{second term}}
\\ & 
\underbrace{
- \frac{1}{ \sqrt{ m^{L-1} d_y} } \beta \sum_{l=1}^L \Pi_{i=l+1}^L \left( \W^{i}_{t-1} + \M_{t-1}^{i} \right)  \W^{l}_{t-1} \Pi_{j=1}^{l-1} \left( \W^{j}_{t-1} + \M_{t-1}^{j} \right)  \X }_{\text{third term}}.
\end{aligned} 
\end{equation}
According to the numbers of momentum terms, we can reformulate (\ref{deep1: lemma5psi1}) as $\E_0 +  \E_1 +  \E_2 + \dots +  \E_L$, where $\E_i$ is composed of the multiplication of $i$ distinct $\M_{t-1}^{l}$ for $1 \leq l \leq L$.
Specifically, $\E_0$ and $\E_1$ are 0 due to
\begin{equation}
\begin{split}
\E_0 & = \underbrace{ \frac{1}{ \sqrt{ m^{L-1} d_y} }\beta(L-1)  \W^{L:1}_{t-1} \X }_{ \text{belongs to the first term} }+ 
 \underbrace{  \frac{1}{ \sqrt{ m^{L-1} d_y} }\beta  \W^{L:1}_{t-1} \X }_{ \text{belongs to the second term} }
 \underbrace{
- \frac{1}{ \sqrt{ m^{L-1} d_y} }\beta L \W^{L:1}_{t-1} \X }_{ \text{belongs to the third term} } = 0
\\
\E_1 & = \underbrace{ - \frac{1}{ \sqrt{ m^{L-1} d_y} } \beta(L-1) \sum_{l=1}^L \W^{L:l+1}_{t-1}  \M_{t-1}^{l} \W^{l-1:1}_{t-1} }_{ \text{belongs to the first term} }
+ \underbrace{ \frac{1}{ \sqrt{ m^{L-1} d_y} }\beta \sum_{l=1}^L \sum_{k \neq l}
\W^{L:k+1}_{t-1}  \M_{t-1}^{k} \W^{k-1:1}_{t-1} }_{ \text{belongs to the third term} }
= 0.
\end{split}
\end{equation}
Then we anaylze the high-order terms.
Consider the $p$-th order term ($p\geq 2$), the first and third terms on (\ref{deep1: lemma5psi1}) provide the coefficients
$-\frac{1}{ \sqrt{ m^{L-1} d_y} }\beta(L-1)$ and  
$\frac{1}{ \sqrt{ m^{L-1} d_y} }\beta(L-p)$
respectively, which results in the coefficient for all the $p$-th order terms as  $\frac{1}{ \sqrt{ m^{L-1} d_y} }\beta(1-p)$.

Applying the bound of the momentum term in (\ref{deepl: lemma5_3}), it has
\begin{eqnarray} \label{deepl: lemma5psi11}
&& \| \frac{1}{ \sqrt{ m^{L-1} d_y} }\beta (L-1)  \W^{L:1}_{t} \X + \frac{1}{ \sqrt{ m^{L-1} d_y} }\beta  \W^{L:1}_{t-1} \X
- \frac{1}{ \sqrt{ m^{L-1} d_y} } \beta \sum_{l=1}^L \W^{L:l+1}_t \W^{l}_{t-1} \W^{l-1:1}_{t} \X \|_F \nonumber\\
 &\overset{}{\leq}& \frac{\beta}{ \sqrt{ m^{L-1} d_y}  } \sum_{j=2}^L \left(j-1\right) {L \choose j} 
\left( \eta 
\frac{336 \| \X \|\sqrt{\kappa}}{\sqrt{d_y}}  \frac{ \theta^{t-1} }{1 -\theta}    \| \U_0 -\Y \|_F  \right)^j 1.25^2( \O(\sqrt{L}))^{j-1}  m^{\frac{L-j}{2}} \| \X \|\nonumber\\ 
&\overset{(a)}{\leq}& 1.25^2 \frac{\beta}{ \O(\sqrt{L}) \sqrt{ m^{L-1} d_y}  } \sum_{j=2}^L L^j 
\left( \eta
\frac{168 \| \X \|\sqrt{\kappa}}{\sqrt{d_y}}  \frac{ \theta^{t-1} }{1 -\theta}    \| \U_0 -\Y \|_F \O(\sqrt{L})\right)^j  m^{\frac{L-j}{2}} \| \X \|
\nonumber\\ 
&\overset{(b)}{\leq}& \frac{1.25^2}{\O(\sqrt{L})} \beta \sqrt{ \frac{ m}{ d_y}  } \| \X \| \sum_{j=2}^L  
\left( \eta
\frac{ 168 L \| \X \|\sqrt{\kappa}}{\sqrt{m d_y}}  \frac{ \theta^{t-1} }{1 -\theta}     \| \U_0 -\Y \|_F \O(\sqrt{L})\right)^j, \nonumber\\ 
& \leq& \frac{1.25^2}{\O(\sqrt{L})} \beta \sqrt{ \frac{ m}{ d_y}  } \| \X \| 
\left( \eta
\frac{168 L \| \X \|\sqrt{\kappa}}{\sqrt{m d_y}}  \frac{ \theta^{t-1} }{1 -\theta}   \| \U_0 -\Y \|_F \O(\sqrt{L}) \right)^2  
\sum_{j=2}^{L-2} 
\left(  0.5 \right)^{j-2}   \nonumber\\ 
& \overset{(c)}{\leq}& \frac{2*1.25^2}{\O(\sqrt{L})} \beta \sqrt{ \frac{ m}{ d_y}  } \| \X \|
\left( \eta
\frac{168 L \| \X \|\sqrt{\kappa}}{\sqrt{m d_y}}  \frac{ \theta^{t-1} }{1 -\theta}   \| \U_0 -\Y \|_F \O(\sqrt{L})\right)^2   \nonumber\\ 
& \overset{(d)}{\leq} & \frac{ \O(\sqrt{L})   }{  \| \X \|}\sqrt{\frac{d_y}{m}}   
\left( 
 \frac{ \theta^{t-1} }{1 -\theta} 90\sqrt{\kappa}   \| \U_0 -\Y \|_F \right)^2 \nonumber\\ 
 &\leq & \frac{1}{90\sqrt{\kappa}} \theta^{2t}\| \U_0 -\Y \|_F,
\end{eqnarray}
where (a) uses ${L \choose j  } \leq \frac{L^j}{j!} $ and (b) uses the same analysis way as (\ref{deepl: lemma5bm}) 
(c) uses $\eta=\frac{d_y}{2*1.2^4L\sigma^2_{max}(\X)}$(d) uses $m \geq C\frac{Ld_yB_0^2\kappa^5}{\|\X\|^2}$ for some sufficently large $C \geq 0$.

For the bound of the second part, it has
\begin{equation}\label{deepl: lemma5psi21}
	\begin{split}
& \frac{\eta\beta}{\sqrt{ m^{L-1} d_y}} \|\sum_{l=1}^L \W^{L:l+1}_{t} \frac{ \partial \ell(\W^{L:1}_{t-1})}{ \partial \W^{l}_{t-1} } \W^{l-1:1}_{t}\X - \sum_{l=1}^L \W^{L:l+1}_{t-1} \frac{ \partial \ell(\W^{L:1}_{t-1})}{ \partial \W^{l}_{t-1} } \W^{l-1:1}_{t-1}\X\|\\
&\leq   \frac{\eta\beta}{\sqrt{ m^{L-1} d_y}} \sum_{l=1}^L \left( \|\underbrace{(\W^{L:l+1}_{t}- \W^{L:l+1}_{t-1})\frac{ \partial \ell(\W^{L:1}_{t-1})}{ \partial \W^{l}_{t-1} } \W^{l-1:1}_{t}\X\|}_{\text{first term}} + \underbrace{\|\W^{L:l+1}_{t-1} \frac{ \partial \ell(\W^{L:1}_{t-1})}{ \partial \W^{l}_{t-1} } ( \W^{l-1:1}_{t}-  \W^{l-1:1}_{t-1})\X\|}_{\text{second term}}  \right).
	\end{split}
\end{equation}
Considering the bound of the first term in (\ref{deepl: lemma5psi21}), it has
\begin{equation}
\begin{split}
	\|(\W^{L:l+1}_{t}- \W^{L:l+1}_{t-1})\frac{ \partial \ell(\W^{L:1}_{t-1})}{ \partial \W^{l}_{t-1} } \W^{l-1:1}_{t}\X\|_F &\leq \|\W^{L:l+1}_{t}- \W^{L:l+1}_{t-1}\| \|\frac{ \partial \ell(\W^{L:1}_{t-1})}{ \partial \W^{l}_{t-1} }\|_F \|\W^{l-1:1}_{t}\X\|.
\end{split}
\end{equation}

Then, it has
\begin{equation}\label{deepl: lemma5psi23}
\begin{split}
\|\W^{L:l+1}_{t}- \W^{L:l+1}_{t-1}\| &\leq \|\W^{L:l+1}_{t}-\W^{L:l+1}_{0} + \W^{L:l+1}_{0}-\W^{L:l+1}_{t-1}\| \\
&\leq \|\W^{L:l+1}_{t}-\W^{L:l+1}_{0}\| +\| \W^{L:l+1}_{0}-\W^{L:l+1}_{t-1}\| \\
&\leq \frac{1}{750\kappa} m^{\frac{L-l}{2}},
\end{split}
\end{equation}
and
\begin{equation}\label{deepl: lemma5psi24}
\|\W^{l-1:1}_{t}\X\| \leq 1.25 m^{\frac{l-1}{2}} \|\X\|,
\end{equation}
where the above two inequalities all use Lemma~\ref{deep1: lemm4.5}.
Based on (\ref{deepl: lemma5psi23}), (\ref{deepl: lemma5psi24}) and (\ref{deepl: lemma5_2}), it has
\begin{eqnarray}\label{deepl: lemma5psi25}
&&\!\!\!\!\!\!\!\!\! \frac{\eta\beta}{\sqrt{ m^{L-1} d_y}} \|\sum_{l=1}^L \W^{L:l+1}_{t} \frac{ \partial \ell(\W^{L:1}_{t-1})}{ \partial \W^{l}_{t-1} } \W^{l-1:1}_{t}\X - \sum_{l=1}^L \W^{L:l+1}_{t-1} \frac{ \partial \ell(\W^{L:1}_{t-1})}{ \partial \W^{l}_{t-1} } \W^{l-1:1}_{t-1}\X\| \nonumber\\
&\leq& \!\!\!\!\frac{\eta\beta}{\sqrt{ m^{L-1} d_y}} \sum_{l=1}^L \left( \|\underbrace{(\W^{L:l+1}_{t} \!-\! \W^{L:l+1}_{t-1})\frac{ \partial \ell(\W^{L:1}_{t-1})}{ \partial \W^{l}_{t-1} } \W^{l-1:1}_{t}\X\|}_{\text{first term}} \!+\! \underbrace{\|\W^{L:l+1}_{t-1} \frac{ \partial \ell(\W^{L:1}_{t-1})}{ \partial \W^{l}_{t-1} } ( \W^{l-1:1}_{t}\!-\!  \W^{l-1:1}_{t-1})\X\|}_{\text{second term}}  \right) \nonumber\\
&\leq& \frac{\eta\beta}{\sqrt{ m^{L-1} d_y}} \sum_{l=1}^L(\frac{1}{750\kappa}m^{\frac{L-l}{2}}\frac{54 \| \X \| \sqrt{\kappa}}{\sqrt{d_y}} \theta^{t-1}  \| \U_0 -\Y \|_F 1.25 m^{\frac{l-1}{2}}\|\X\| \nonumber\\
&+& 1.25 m^{\frac{L-l}{2}}\frac{54 \| \X \| \sqrt{\kappa}}{\sqrt{d_y}} \theta^{t-1} \| \U_0 -\Y \|_F\frac{1}{750\kappa}m^{\frac{l-1}{2}}\|\X\|) \nonumber\\
&\leq& \frac{\eta\beta}{\sqrt{  d_y}} \sum_{l=1}^L(\frac{9\|\X\|^2}{50\sqrt{\kappa d_y}} \|\U_0-Y\|_F \theta^{t-1}) \leq \frac{1}{23\sqrt{\kappa}} \theta^{t-1}\|\U_0-Y\|_F \leq \frac{2}{23\sqrt{\kappa}}\theta^t \|\U_0 -\Y\|_F,
\end{eqnarray}
where the last inequality uses  $\beta < \theta^2$ and $\theta = 1 - \frac{1}{2\sqrt{\kappa}} \geq 1/2$.

Combining (\ref{deepl: lemma5psi11}) and (\ref{deepl: lemma5psi25}), it has
\begin{equation}\label{deepl: psi}
\begin{split}
	\|\psib_t\| \leq \frac{1}{90\sqrt{\kappa}} \theta^{2t}\| \U_0 -\Y \|_F + \frac{2}{23\sqrt{\kappa}}\theta^t \|\U_0 -\Y\|_F
\end{split}
\end{equation}

Finally we analyze the term $\| \iotab_t \| \leq \|\eta(1+\beta) (\H_t - \H_0) \xib_t\| + \|\eta\beta(\H_{t-1} - \H_0)\xib_{t-1}\|$, which is closely related to the bound of $\|(\H_i - \H_0)\xib_i\|$ for $i\leq t$.
It has
\begin{eqnarray}  \label{deepl: lemma5iota1}
&&\|(\H_i - \H_0)\xib_i\|\nonumber\\ 
& =&
\frac{1}{m^{L-1} d_y} 
\| \sum_{l=1}^L \W^{L:l+1}_i (\W^{L:l+1}_i)^\top ( \U_i -\Y) (\W^{l-1:1}_i \X)^\top \W^{l-1:1}_i \X \nonumber \\
&-&
\sum_{l=1}^L \W^{L:l+1}_0 (\W^{L:l+1}_0)^\top ( \U_i -\Y) (\W^{l-1:1}_0 \X)^\top \W^{l-1:1}_0 \X\|_F \nonumber\\ 
& \leq &
\frac{1 }{m^{L-1} d_y} 
\sum_{l=1}^L  
\| \W^{L:l+1}_i (\W^{L:l+1}_i)^\top ( \U_i -\Y) (\W^{l-1:1}_i \X)^\top \W^{l-1:1}_i  \X \nonumber\\
& -&
\W^{L:l+1}_0  (\W^{L:l+1}_0 )^\top ( \U_i -\Y) (\W^{l-1:1}_0 \X)^\top \W^{l-1:1}_0 \X\|_F \nonumber\\ 
& \leq &
\frac{1}{m^{L-1} d_y} 
\sum_{l=1}^L  \big( 
\underbrace{ 
\|\left( \W^{L:l+1}_i (\W^{L:l+1}_i)^\top - \W^{L:l+1}_0 (\W^{L:l+1}_0)^\top \right)  ( \U_i -\Y) (\W^{l-1:1}_i \X)^\top \W^{l-1:1}_i  \X \|_F }_{\text{ first term} }\nonumber\\ 
&+& 
\underbrace{ 
\| \W^{L:l+1}_0 (\W^{L:l+1}_0)^\top  ( \U_i -\Y) \left( \W^{l-1:1}_i \X)^\top \W^{l-1:1}_i \X  - (\W^{l-1:1}_0 \X)^\top \W^{l-1:1}_0 \X   \right)  \|_F \big)
}_{\text{ second term} }.
\end{eqnarray}

For the first term, it has
\begin{eqnarray} \label{deepl: lemma5iota2}
&&\underbrace{ \|\left( \W^{L:l+1}_i (\W^{L:l+1}_i)^\top - \W^{L:l+1}_0 (\W^{L:l+1}_0)^\top \right)  ( \U_i -\Y) (\W^{l-1:1}_i \X)^\top \W^{l-1:1}_i \X  \|_F }_{\text{ first term} }
\nonumber\\ & \leq &
\| \W^{L:l+1}_i (\W^{L:l+1}_i)^\top - \W^{L:l+1}_0 (\W^{L:l+1}_0)^\top \|
\| \U_i -\Y \|_F \|  (\W^{l-1:1}_t \X)^\top \W^{l-1:1}_t \X  \|.\nonumber\\
\end{eqnarray}

Using $\|\W^{L:i}_t - \W^{L:i}_0\| \leq \frac{1}{1500 \kappa } ( \sqrt{m} )^{L-i+1}$ as proved in (\ref{deep1: lemm4.51}), we have
\begin{eqnarray} \label{deepl: lemma5iota3}
& &\| \W^{L:l+1}_i (\W^{L:l+1}_i)^\top - \W^{L:l+1}_0 (\W^{L:l+1}_0)^\top \| \nonumber\\ 
& \leq & \| (\W^{L:l+1}_i - \W^{L:l+1}_0)(\W^{L:l+1}_i)^\top  +  \W^{L:l+1}_i (\W^{L:l+1}_i - \W^{L:l+1}_0)^\top \nonumber \\
&+& (\W^{L:l+1}_i - \W^{L:l+1}_0) ( \W^{L:l+1}_i - \W^{L:l+1}_0)^\top \| \nonumber\\ 
& \leq & 2 \| \W^{L:l+1}_i - \W^{L:l+1}_0 \| \cdot \sigma_{\max} (\W^{L:l+1}_t)  +
\| \W^{L:l+1}_i - \W^{L:l+1}_0\|^2 \nonumber\\ 
& \leq& (\frac{2.5}{1500\kappa} + \frac{1}{(1500\kappa)^2}) m^{L-l},
\end{eqnarray}
where the last inequality uses Lemma~\ref{deep1: lemm4.5}.

For $\|  (\W^{l-1:1}_i \X)^\top \W^{l-1:1}_i \X  \|$, with 
Lemma~\ref{deep1: lemm4.5},
it has
\begin{equation} \label{deepl: lemma5iota4}
\| (\W^{l-1:1}_i \X)^\top \W^{l-1:1}_i \X \| 
\leq \left( \sigma_{\max}( \W^{l-1:1}_i \X) \right)^2 \leq
\left( 1.25 m^{\frac{l-1}{2}} \sigma_{\max}(\X) \right)^2.
\end{equation}
Thus, 
\begin{equation} \label{deepl: lemma5iota5}
\begin{split}
&\underbrace{ \|\left( \W^{L:l+1}_t (\W^{L:l+1}_t)^\top - \W^{L:l+1}_0 (\W^{L:l+1}_0)^\top \right)  ( \U_i -\Y) (\W^{l-1:1}_t \X)^\top \W^{l-1:1}_t \X  \|_F }_{\text{ first term} }
\\ & 
\leq (\frac{2.5}{1500\kappa} + \frac{1}{(1500\kappa)^2}) m^{L-l} 
\left( 1.25 m^{\frac{l-1}{2}} \sigma_{\max}(\X) \right)^2 \| \U_i -\Y \|_F
\\ & 
\leq \frac{\sigma_{\min}^2(\X)}{1940} m^{L-1}  \| \U_i -\Y \|_F,
\end{split}
\end{equation}
where the last inequality uses $\kappa = (1.2/0.8)^4\frac{\sigma_{\max}^2(\X)}{\sigma_{\min}^2(\X) }$.

Following a similar approach, the second part can be bounded by
\begin{eqnarray} \label{deepl: lemma5iota6}
&& \underbrace{ 
\| (\W^{L:l+1}_0 (\W^{L:l+1}_0)^\top  ( \U_i -\Y) \left( \W^{l-1:1}_i \X)^\top \W^{l-1:1}_i \X  - (\W^{l-1:1}_0 \X)^\top \W^{l-1:1}_0 \X   \right)  \|_F \big)
}_{\text{ second term} }
\nonumber\\ 
& \leq& 
\| (\W^{L:l+1}_0 (\W^{L:l+1}_0)^\top \| \| \U_i -\Y \|_F 
\| (\W^{l-1:1}_i \X)^\top \W^{l-1:1}_i \X  - (\W^{l-1:1}_0 \X)^\top \W^{l-1:1}_0 \X   \|
\nonumber\\
&\leq &(1.25m^{\frac{L-l}{2}})^2 (\frac{2.5}{1500\kappa} + \frac{1}{(1500\kappa)^2}) m^{l-1}\sigma_{max}^2(\X) \|\U_i -\Y\|_F
\nonumber\\
&\leq& \frac{\sigma_{\min}^2(\X)}{1940} m^{L-1}  \| \U_i -\Y \|_F.
\end{eqnarray}
Combining (\ref{deepl: lemma5iota5}) and (\ref{deepl: lemma5iota6}), it has
\begin{eqnarray}  \label{deepl: lemma5iota7}
&&\|(\H_i - \H_0)\xib_i\| \nonumber\\ 
& \leq &\!\!\!\!
\frac{1}{m^{L-1} d_y} 
\sum_{l=1}^L  \big( 
\underbrace{ 
\|\left( \W^{L:l+1}_i (\W^{L:l+1}_i)^\top - \W^{L:l+1}_0 (\W^{L:l+1}_0)^\top \right)  ( \U_i -\Y) (\W^{l-1:1}_i \X)^\top \W^{l-1:1}_i  \X \|_F }_{\text{ first term} }
\nonumber\\ &   + &
\underbrace{ 
\| \W^{L:l+1}_0 (\W^{L:l+1}_0)^\top  ( \U_i -\Y) \left( \W^{l-1:1}_i \X)^\top \W^{l-1:1}_i \X  - (\W^{l-1:1}_0 \X)^\top \W^{l-1:1}_0 \X   \right)  \|_F \big)
}_{\text{ second term} }
\nonumber\\&\leq& \frac{1}{m^{L-1} d_y} 
\sum_{l=1}^L 2\frac{\sigma_{\min}^2(\X)}{1940} m^{L-1}  \| \U_i -\Y \|_F \leq \frac{12\sqrt{2} L\sigma_{\min}^2(\X)\sqrt{\kappa}}{485 d_y}\theta^i  \|\U_0 -\Y\|_F.
\end{eqnarray}

Then, it has
\begin{eqnarray} \label{deepl: iota}
\| \iotab_t \| &\leq &\|\eta(1+\beta) (\H_t - \H_0) \xib_t\| + \|\eta\beta(\H_{t-1} - \H_0)\xib_{t-1}\| \nonumber\\
& \leq & \frac{12\sqrt{2}}{(0.8)^4 485\sqrt{\kappa}}\theta^t  \|\U_0 -\Y\|_F + \frac{12\sqrt{2}\theta}{2(0.8)^4 485\sqrt{\kappa}}\theta^t  \|\U_0 -\Y\|_F \nonumber\\
&\leq& \frac{5}{39\sqrt{\kappa}}\theta^t  \|\U_0 -\Y\|_F,
\end{eqnarray} 
where the second inequality uses $\eta =\frac{d_y}{2*1.2^4L\sigma_{max}^2(\X)}$, $\beta \leq \theta^2$, and $\kappa = (1.2/0.8)^4 \frac{\sigma_{max}^2(\X)}{\sigma_{min}^2(\X)}$.

Combining (\ref{deepl: phi}), (\ref{deepl: psi}), and  (\ref{deepl: iota}),  
it has
\begin{eqnarray}
\| \varphib_t \| & \leq& 
\| \phib_t \| + \|\psib_t \| + \| \iotab_t \| \nonumber\\ 
& \leq &\!\!\!
\frac{1}{180\sqrt{\kappa}} \theta^{2t}\| \U_0 -\Y \|_F + \frac{1}{90\sqrt{\kappa}} \theta^{2t}\| \U_0 -\Y \|_F
  + \frac{2}{23\sqrt{\kappa}}\theta^{t} \|\U_0 -\Y\|_F  +\frac{5}{39\sqrt{\kappa}}\theta^t \|\U_0 -\Y\|_F
\nonumber\\
&\leq& \frac{1}{60\sqrt{\kappa}} \theta^{2t}\| \U_0 -\Y \|_F+
 \frac{5}{23\sqrt{\kappa}} \theta^t \|\U_0 -\Y\|_F .
\end{eqnarray}
\end{proof}

\subsection{Proof of Lemma~~\ref{deepl: lemma_distance}}
\begin{proof}
We have 
\begin{eqnarray}
 \| \W^{l}_{t+1} - \W^{l}_0 \|_F 
& \overset{(a)}{\leq} &
 \sum_{s=0}^t \|  \M_{s,l} \|_F \nonumber\\ 
& \overset{(b)}{=} &
\eta \sum_{s=0}^t
\| \sum_{\tau=0}^s \beta^{s-\tau}   \left( \frac{ \partial \ell(\W^{L:1}_{\tau})}{ \partial \W^{l}_{\tau} } + \beta(\frac{ \partial \ell(\W^{L:1}_{\tau})}{ \partial \W^{l}_{\tau} } - \frac{ \partial \ell(\W^{L:1}_{\tau-1})}{ \partial \W^{l}_{\tau-1} })\right)  \|_F \nonumber \\
&\leq &
\eta \sum_{s=0}^t \sum_{\tau=0}^s (1+\beta)\beta^{s-\tau} 
\| \frac{ \partial \ell(\W^{L:1}_{\tau})}{ \partial \W^{l}_{\tau} }   \|_F + \eta \sum_{s=0}^t \sum_{\tau=0}^s \beta^{s-\tau+1} 
\| \frac{ \partial \ell(\W^{L:1}_{\tau-1})}{ \partial \W^{l}_{\tau-1} }   \|_F
\nonumber\\ 
&\overset{(c)}{\leq}& \eta(1+\beta + \theta)\frac{54\| \X \|\sqrt{\kappa}}{\sqrt{d_y}}\| \U_0 -\Y \|_F \sum_{s=0}^t \sum_{\tau=0}^s \theta^{2(s-\tau)} 
 \theta^{\tau}  .
\nonumber\\
&\leq& 
\frac{162\| \X \|\sqrt{\kappa} \eta }{\sqrt{d_y}}\| \U_0 -\Y \|_F \frac{1}{(1 - \theta)^2}
\nonumber\\ 
&\overset{(d)}{\leq}& \frac{792\| \X \| B_0 \sqrt{d_y\kappa}}{ L \sigma_{\min}^2(\X) }  ,
\end{eqnarray}
where (a) is by recursively using (\ref{procudure: NAG_2}), (b) uses $\M_t^l= \sum_{s=0}^t \beta^{t-s} \frac{ \partial \ell(W_{L:1})}{ \partial \W^{l}_s }$, (c) uses $\|\frac{ \partial \ell(W_{L:1})}{ \partial \W^{l}_s }\|_F = 
\frac{108\| \X \|\sqrt{\kappa}}{\sqrt{d_y}} \theta^{s}  \| \U_0 -\Y \|_F$ as (\ref{deepl: lemma5_2}) and $\beta \leq \theta^2$,
(d) uses $\frac{1}{(1-\theta)^2} = \frac{2}{\eta \lambda_{min}}$,
 the upper-bound $B_0 \geq \| \U_0 -\Y \|$ defined in Lemma~\ref{deepl: lemma2.5} and $\lambda_{min}=(0.8)^4 L \sigma^2_{min}(\X) / d_y$.
\end{proof}

\subsection{Proof of Theorem~\ref{thm:LinearNet}} \label{sec:linear}
\begin{proof} 
We prove the theorem by induction.
The base case $s = 0$ holds.
Assume $
\left\|
\begin{bmatrix}
\xib_{s} \\
\xib_{s-1} 
\end{bmatrix}
\right\|
\leq \theta^{s} 24\sqrt{\kappa}  \left\|
\begin{bmatrix}
\xib_{0} \\
\xib_{-1} 
\end{bmatrix}
\right\|
$ holds for $s \leq t-1$.

Based on Lemma~\ref{deepl: lemma1}, it is noted that
\begin{eqnarray}
\begin{bmatrix}
\xib_{t} \\
\xib_{t-1} 
\end{bmatrix}
= 
\mathbf{M}
\begin{bmatrix}
\xib_{t-1} \\
\xib_{t-2} 
\end{bmatrix}
+
\begin{bmatrix}
\varphib_{t-1} \\ 0_{d_y n}
\end{bmatrix}, \nonumber
\end{eqnarray}
where $\G = \begin{bmatrix}
(1+\beta)(\I_{d_y n} - \eta \H_0^{lin}) & \beta ( -\I_{d_y n} + \eta \H_{0}^{lin}  )   \\
\I_{d_y n} & \0_{d_y n} 
\end{bmatrix}$.
By recursively using above equation, it has
\begin{eqnarray}
\label{deepl: theorem_1}
	\begin{bmatrix}
\xib_{t} \\
\xib_{t-1} 
\end{bmatrix}	= \G^t \begin{bmatrix}
\xib_{0} \\
\xib_{-1} 
\end{bmatrix} + \sum_{s=0}^{t-1} \G^{t-s-1} \begin{bmatrix}
\varphib_s \\
0_{d_y n}
\end{bmatrix}.
\end{eqnarray}
From Lemma~\ref{supportlemma1} and Lemma~\ref{supportlemma2}, it has the bound for the first term on the right hand side of (\ref{deepl: theorem_1}) as
\begin{eqnarray}
\label{deepl: theorem_2}
\left\|\G^t \begin{bmatrix}
\xib_{0} \\
\xib_{-1} 
\end{bmatrix}  \right\| \leq 12\sqrt{\kappa}\rho^t  \left\|\begin{bmatrix}
\xib_{0} \\
\xib_{-1} 
\end{bmatrix} \right\|,
\end{eqnarray}
where $\rho = 1 - \frac{2}{3\sqrt{\kappa}}$.

Applying the inductive hypothesis and Lemma~\ref{deepl: lemma_distance}, it has the upper bound for the distance $\|\W_i^l - \W_0^l\| \leq R^{lin}$ for any $i \leq t$ and $l \in [m]$.
In turn, we can bound the second term on the right hand side of (\ref{deepl: theorem_1}) as
\begin{eqnarray} \label{deepl: theorem_3}
 \left\| \sum_{s=0}^{t-1} \G^{t-1-s} \begin{bmatrix}
\varphib_{s} \\ 0 \end{bmatrix} \right\|
& \overset{(a)}{ \leq}& 
\sum_{s=0}^{t-1} 12\sqrt{\kappa}\rho^{t-1-s}  \| \varphib_s \| \nonumber\\
& \overset{(b)}{ \leq} & \sum_{s=0}^{t-1} \rho^{t-1-s}12\sqrt{\kappa}(\frac{1}{60\sqrt{\kappa}} \theta^{2s}\| \U_0 -\Y \|_F + \frac{5}{23\sqrt{\kappa}}  \theta^s \|\U_0 -\Y\|_F)
\nonumber\\ 
& \overset{(c)}{ \leq} &
12\sqrt{\kappa}  \theta^t (\frac{\sqrt{2}}{20} + \frac{15\sqrt{2}}{23})\left\| \begin{bmatrix} \xib_0 \\ \xib_{-1} \end{bmatrix} \right\|\nonumber\\
& \overset{}{ \leq}&
12\sqrt{\kappa}  \theta^{t}    \left\| \begin{bmatrix} \xib_0 \\ \xib_{-1} \end{bmatrix} \right\|,
\end{eqnarray}
where (a) uses Lemma~\ref{supportlemma1} and Lemma~\ref{supportlemma2},
(b) uses the bound of $\|\varphib_{s}\|$ in Lemma~\ref{deepl: lemma2},
 (c) uses $\sum_{s=0}^{t-1} \rho^{t-1-s} \theta^s = \theta^{t-1} \sum_{s=0}^{t-1} \left( \frac{\rho}{\theta}  \right)^{t-1-s} \leq \theta^{t-1} \frac{1-(\rho/\theta)^t}{1-\rho/\theta} \leq 6\sqrt{\kappa}\theta^t $ and $\|\xib_{-1}\| = \|\xib_0\|$.
 
Combining (\ref{deepl: theorem_2}) and (\ref{deepl: theorem_3}), it completes the proof.

\end{proof}
\clearpage

\section{Deep linear Resnet} \label{resnet}

\subsection{Proof of Lemma~\ref{deepres: lemm1}}
\begin{proof}
According to the update rule of NAG,
it has
\begin{equation} \label{deepres:multiply_params}
\tilde{\W}^{L:1}_{t+1}  = \Pi_{l=1}^L \left( \tilde{\W}^{l)}_{t} + \M_t^l \right)
=   \tilde{\W}^{L:1}_{t}  + \sum_{l=1}^L  \tilde{\W}^{L:l+1}_{t}  \M_t^l   \tilde{\W}^{l-1:1}_{t}  + \Phi_t,
\end{equation}
where 
$\Phi_t$ contains all the high-order multiplication of momentum terms, i.e. second-order $\M_{t, i}\M_{t, j}$ for $\forall i \neq j$ and higher terms.
 Based on the equivalent update expression of NAG and 
$ \M_t^l = - \eta \frac{ \partial \ell(\W^{L:1}_t)}{ \partial \W^{l}_t } - \eta\beta(\frac{ \partial \ell(\W^{L:1}_t)}{ \partial \W^{l}_t } - \frac{ \partial \ell(\W^{L:1}_{t-1})}{ \partial \W^{l}_{t-1} })+
\beta ( \W^{l}_t - \W^{l}_{t-1} )$
we can rewrite (\ref{deepres:multiply_params}) as
\begin{equation}
\begin{aligned}
\tilde{\W}^{L:1}_{t+1} 
& = \tilde{\W}^{L:1}_t - \eta(1+\beta) \sum_{l=1}^L \tilde{\W}^{L:l+1}_t \frac{ \partial \ell(\W^{L:1}_t)}{ \partial \W^{l}_t } \tilde{\W}^{l-1:1}_t +  \eta \beta \sum_{l=1}^L \tilde{\W}^{L:l+1}_t \frac{ \partial \ell(\W^{L:1}_{t-1})}{ \partial \W^{l}_{t-1} } \tilde{\W}^{l-1:1}_t \nonumber \\
&+ \sum_{l=1}^L \tilde{\W}^{L:l+1}_t \beta ( \W^{l}_t - \W^{l}_{t-1} ) \tilde{\W}^{l-1:1}_t + \Phi_t
\\ & = \tilde{\W}^{L:1}_t - \eta(1+\beta) \sum_{l=1}^L \tilde{\W}^{L:l+1}_t \frac{ \partial \ell(\W^{L:1}_t)}{ \partial \W^{l}_t } \tilde{\W}^{l-1:1}_t + \beta ( \tilde{\W}^{L:1}_t - \tilde{\W}^{L:1}_{t-1} ) \\
&+ \eta\beta \sum_{l=1}^L \tilde{\W}^{L:l+1}_{t-1} \frac{ \partial \ell(\W^{L:1}_{t-1})}{ \partial \W^{l}_{t-1} } \tilde{\W}^{l-1:1}_{t-1}
+ (L-1) \beta \tilde{\W}^{L:1}_{t} + \beta  \tilde{\W}^{L:1}_{t-1}
- \beta \sum_{l=1}^L \tilde{\W}^{L:l+1}_t \tilde{\W}^{l}_{t-1} \tilde{\W}^{l-1:1}_{t} \nonumber\\ 
&+ \eta\beta(\sum_{l=1}^L \tilde{\W}^{L:l+1}_{t} \frac{ \partial \ell(\W^{L:1}_{t-1})}{ \partial \W^{l}_{t-1} } \tilde{\W}^{l-1:1}_{t} - \sum_{l=1}^L \tilde{\W}^{L:l+1}_{t-1} \frac{ \partial \ell(\W^{L:1}_{t-1})}{ \partial \W^{l}_{t-1} } \tilde{\W}^{l-1:1}_{t-1}) + \Phi_t. 
\end{aligned}
\end{equation}
Left multiplying the above equality with $\B$ and right with $\A\X$, we get
\begin{eqnarray}
\label{deepres: output}
\U_{t+1} & =& \U_t - \eta(1+\beta) \sum_{l=1}^L \B\tilde{\W}^{L:l+1}_t  (\B\tilde{\W}^{L:l+1}_t)^\top
( \U_t -\Y )   (\tilde{\W}^{l-1:1}_t \A\X)^\top  \tilde{\W}^{l-1:1}_t \A\X \nonumber\\
& +& \beta  (\U_t - \U_{t-1})  +  \eta\beta \sum_{l=1}^L \B\tilde{\W}^{L:l+1}_{t-1}  (\B\tilde{\W}^{L:l+1}_{t-1})^\top
( \U_{t-1} -\Y )   (\tilde{\W}^{l-1:1}_{t-1} \A\X)^\top  \tilde{\W}^{l-1:1}_{t-1} \A\X  \nonumber\\
& +& \B \left( (L-1) \beta \tilde{\W}^{L:1}_{t} + \beta  \tilde{\W}^{L:1}_{t-1}
- \beta \sum_{l=1}^L \tilde{\W}^{L:l+1}_t \tilde{\W}^{l}_{t-1} \tilde{\W}^{l-1:1}_{t} \right) \A\X
+ \B\Phi_t \A\X \nonumber \\
& +& \eta\beta \B \left(\sum_{l=1}^L \tilde{\W}^{L:l+1}_{t} \frac{ \partial \ell(\W^{L:1}_{t-1})}{ \partial \W^{l}_{t-1} } \tilde{\W}^{l-1:1}_{t} - \sum_{l=1}^L \tilde{\W}^{L:l+1}_{t-1} \frac{ \partial \ell(\W^{L:1}_{t-1})}{ \partial \W^{l}_{t-1} } \tilde{\W}^{l-1:1}_{t-1} \right)\A\X.
\end{eqnarray}
With $\text{vec}(\A\bm{C}\B) = (\B^\top \otimes \A) \text{vec}(\bm{C})$, (\ref{deepres: output}) can be vectorized as
\begin{eqnarray} \label{deepres:residual_error}
&&\v(\U_{t+1}) - \v(\U_t)\nonumber\\
&  =& - \eta(1+\beta) \H_t \v( \U_t -\Y ) +
\beta  \left( \v(\U_{t}) - \v(\U_{t-1})  \right) + \eta\beta \H_{t-1}\v(\U_{t-1} -\Y)
\nonumber\\ &  + &
\v\left(  \B \left( (L-1) \beta \tilde{\W}^{L:1}_{t} + \beta  \tilde{\W}^{L:1}_{t-1}
- \beta \sum_{l=1}^L \tilde{\W}^{L:l+1}_t \tilde{\W}^{l}_{t-1} \tilde{\W}^{l-1:1}_{t} \right) \A\X \right)
\nonumber\\ &  + &\v\left(\eta\beta \B \left(\sum_{l=1}^L \tilde{\W}^{L:l+1}_{t} \frac{ \partial \ell(\W^{L:1}_{t-1})}{ \partial \W^{l}_{t-1} } \tilde{\W}^{l-1:1}_{t} - \sum_{l=1}^L \tilde{\W}^{L:l+1}_{t-1} \frac{ \partial \ell(\W^{L:1}_{t-1})}{ \partial \W^{l}_{t-1} } \tilde{\W}^{l-1:1}_{t-1} \right) \A\X\right)
\nonumber\\&  + &
 \v( \B\Phi_t \A\X),
\end{eqnarray}
where 
\begin{equation}
\H_t^{res} =  \sum_{l=1}^L \left[ \left( (\tilde{\W}^{l-1:1}_t \A\X)^\top (\tilde{\W}^{l-1:1}_t \A\X)
 \right)  \otimes \left( \B\tilde{\W}^{L:l+1}_t (\B\tilde{\W}^{L:l+1}_t)^\top \right)   \right].
\end{equation}

Then (\ref{deepres:residual_error}) can be rewritten as
\begin{equation} \label{eq:residual_error2}
\begin{split}
\begin{bmatrix}
\xib_{t+1} \\
\xib_{t} 
\end{bmatrix}
& = 
\begin{bmatrix}
(1+\beta)(\I_{d_y n} - \eta \H_t^{res}) & \beta ( -\I_{d_y n} + \eta \H_{t-1}^{res}  ) \\
\I_{d_y n} & \0_{d_y n} 
\end{bmatrix}
\begin{bmatrix}
\xib_{t} \\
\xib_{t-1} 
\end{bmatrix}
+
\begin{bmatrix}
\phib_t + \psib_t \\ \0_{d_y n}
\end{bmatrix}
\\ & = 
\begin{bmatrix}
(1+\beta)(\I_{d_y n} - \eta \H_0^{res}) & \beta ( -\I_{d_y n} + \eta \H_{0}^{res}  )   \\
\I_{d_y n} & \0_{d_y n} 
\end{bmatrix}
\begin{bmatrix}
\xib_{t} \\
\xib_{t-1} 
\end{bmatrix}
+
\begin{bmatrix}
\varphib_t \\ \0_{d_y n}
\end{bmatrix}
,
\end{split}
\end{equation}
where $\varphib_t = \phib_t + \psib_t + \iotab_t \in \reals^{d_y n}$.

\end{proof}

\begin{lemma} \label{lem:DLresnet_init}{(Proposition 3.3 in~\cite{DBLP:conf/iclr/ZouLG20})}
 By the initialization as shown in Section, with $m \geq C(d_{x}+d_y+\log(1/\delta))$ for some constant $C$, with probability at least $1-\delta$, it has
\begin{eqnarray}
 0.9\alpha\sqrt{m}\leq \sigma_{min}(\A) \leq \sigma_{max}(\A) \leq 1.1\alpha\sqrt{m} \!\!&,&\!\!  0.9\beta\sqrt{m}\leq \sigma_{min}(\B) \leq \sigma_{max}(\B) \leq 1.1\beta\sqrt{m} \nonumber\\
 \lambda_{\min}(\H_0^{res})  \geq  (0.9)^4 L \alpha^2\gamma^2 m^2 \sigma^2_{min}(\X) \!\!&,& \!\!
   \lambda_{\max}(\H_0^{res})  \leq (1.1)^4 L \alpha^2\gamma^2 m^2 \sigma^2_{\max}(\X) ,\nonumber\\ 
   \kappa(\H_0^{res}) \leq \frac{1.1^4\sigma^2_{max}(\X)}{0.9^4\sigma^2_{min}(\X)} \quad,\quad \ell(\W_0) \leq B_0^2 \!\!\!&=&\!\!\! (6.05 \alpha^2\gamma^2 d_y m \log(2n/\delta) + \|\W^*\|^2)\|\X\|_F^2 . \nonumber
\end{eqnarray}
\end{lemma}

\subsection{Proof of Lemma~\ref{deepres: lemm2}}

\begin{proof}
By Lemma~\ref{deepres: lemm1},
it has $\varphib_t = \phib_t + \psib_t  + \iotab_t \in \reals^{d_y n}$, where
\begin{eqnarray}
 \phib_t & = \v( \B\Phi_t \A\X)
\text{ , with } 
\Phi_t    = \Pi_l ( \tilde{\W}^{l}_t + \M_t^l )
- \tilde{\W}^{L:1}_t  - \sum_{l=1}^L \tilde{\W}^{L:l+1}_t \M_t^l \tilde{\W}^{l-1:1}_t, \nonumber
\end{eqnarray}
and
\begin{eqnarray}
 \psib_t&=& 
\v\left(  \B ( (L-1) \beta \tilde{\W}^{L:1}_{t} + \beta  \tilde{\W}^{L:1}_{t-1}
- \beta \sum_{l=1}^L \tilde{\W}^{L:l+1}_t \tilde{\W}^{l}_{t-1} \tilde{\W}^{l-1:1}_{t} ) \A\X \right) \nonumber\\
& +& \v\left(\eta\beta \B (\sum_{l=1}^L \tilde{\W}^{L:l+1}_{t} \frac{ \partial \ell(\W^{L:1}_{t-1})}{ \partial \W^{l}_{t-1} } \tilde{\W}^{l-1:1}_{t} - \sum_{l=1}^L \tilde{\W}^{L:l+1}_{t-1} \frac{ \partial \ell(\W^{L:1}_{t-1})}{ \partial \W^{l}_{t-1} } \tilde{\W}^{l-1:1}_{t-1}) \A\X\right). \nonumber
\end{eqnarray}
and
\begin{eqnarray}
& \iotab_t=  -\eta(1+\beta) (\H_t^{res} - \H_0^{res}) \xib_t + \eta\beta(\H_{t-1}^{res} - \H_0^{res})\xib_{t-1}. \nonumber
\end{eqnarray}

If we can bound $\| \phib_t \|$, $\| \psib_t \|$, and $\| \iotab_t\|$ respectively, then the bound of $\| \varphib_t \| $ can be derived by the triangle inequality.
\begin{equation} \label{eq:var}
\| \varphib_t \| \leq \| \phib_t \| + \|\psib_t \| + \| \iotab_t \|. 
\end{equation}

Let us first provide the bound of $\| \phib_t \|$.
Note that $\Phi_t$ is the sum of all the high-order momentum terms in the product,
\begin{equation}
\tilde{\W}^{L:1}_{t+1} = \Pi_l \left( \tilde{\W}^{l}_t + \M_t^l \right)
= \tilde{\W}^{L:1}_t + \sum_{l=1}^L \tilde{\W}^{L:l+1}_t \M_t^l \tilde{\W}^{l-1:1} + \Phi_t.
\end{equation}
Using the inductive hypothesis, we can bound the gradient norm of each layer as
\begin{equation} \label{eq:DLresnetnorm-linear}
\begin{split}
 \| \frac{ \partial \ell(\W^{L:1}_s)}{ \partial \W^{l}_s } \|_F &= \|(\B\tilde{\W}_s^{L:l+1})^{\top}(\U_s -\Y)(\tilde{\W}_s^{l-1:1}\A\X)^{\top}\|_F\\
 &
\overset{(a)}{\leq} (1+R^{res})^{L-1}\|\A\|\|\B\|\|\X\|\|U_s -\Y\|_F \\
& \overset{(b)}{\leq} 36\|\A\|\|\B\|\|\X\|\sqrt{\kappa}  \theta^s  \|\U_0 -\Y\|_F,\\
\end{split}
\end{equation}
where (a) uses $\|\tilde{\W}_s^i\| = \|\I + \W_s^i\| \leq 1 + \|\W_s^i\| \leq 1 + \|\W_s^i\| \leq 1 + R^{res}$ for any $s \leq t$, 
(b) uses the induction hypothesis and $(1+R^{res})^{L-1} \leq (1+R^{res})^L \leq exp(1/(2000\kappa)) \leq 1 + \frac{e-1}{2000\kappa} \leq 1.001$.

Thus the momentum term of each layer can be bounded as
\begin{eqnarray}
\label{eq:Dlresnet_M}
	\|\M_t^l\| &=& 
\|-\eta\sum_{s=0}^t \beta^{t-s}\{ \frac{ \partial \ell(\W^{L:1}_s)}{ \partial \W^{l}_s } + \beta(\frac{ \partial \ell(\W^{L:1}_s)}{ \partial \W^{l}_s } - \frac{ \partial \ell(\W^{L:1}_{s-1})}{ \partial \W^{l}_{s-1} })\} \| \nonumber\\
&\leq& \eta(1+\beta)\sum_{s=0}^t\|\beta^{t-s} \frac{ \partial \ell(\W^{L:1}_s)}{ \partial \W^{l}_s }\| + \eta\beta \sum_{s=0}^t\|\beta^{t-s} \frac{ \partial \ell(\W^{L:1}_{s-1})}{ \partial \W^{l}_{s-1} }\| \nonumber\\
&\leq& 36\|\A\|\|\B\|\|\X\|\sqrt{\kappa} \eta \| \U_0 -\Y \|_F \left((1+\beta)\sum_{s=0}^t \beta^{t-s} \theta^s  + \beta \sum_{s=0}^t \beta^{t-s} \theta^{s-1} \right) \nonumber\\
&\overset{(a)}{\leq}& 
36\|\A\|\|\B\|\|\X\|\sqrt{\kappa} \eta \| \U_0 -\Y \|_F(1+\beta+\theta)\frac{\theta^t(1-\theta^{t+1})}{1-\theta} \nonumber\\
&\overset{(b)}{\leq}& 108\|\A\|\|\B\|\|\X\|\sqrt{\kappa} \eta\| \U_0 -\Y \|_F \frac{\theta^t}{1 - \theta},
\end{eqnarray}
where (a) uses $\beta \leq \theta^2$, (b) uses $\beta, \theta \leq 1$.

Combining all these pieces together, we can bound 
$\| \B \Phi_t \A\X \|_F$ as
\begin{eqnarray}  
&& \| \B \Phi_t \A\X \|_F
\nonumber\\ &
\overset{(a)}{\leq}& \|\B\| \sum_{j=2}^L {L \choose j} 
\left( 108\|\A\|\|\B\|\|\X\| \sqrt{\kappa} \eta \| \U_0 -\Y \|_F \frac{\theta^t}{1 - \theta} \right)^j (1+R^{res})^{L-j}\|\A\|\| \X \|
\nonumber\\ &
\overset{(b)}{\leq}&  \sum_{j=2}^L L^j 
\left( \eta
108\|\A\| \|\B\| \|\X\| \sqrt{\kappa} \frac{ \theta^{t} }{1 -\theta}  \| \U_0 -\Y \|_F\right)^j  (1+R^{res})^{L-j} \|\A\|\|\B\|\| \X \|
\nonumber\\ &
\leq & 1.001\|\A\|\|\B\| \| \X \| \sum_{j=2}^L  
\left(L\frac{108\|\A\| \|\B\| \|\X\|\sqrt{\kappa}}{1+R^{res}} \eta
  \frac{ \theta^{t} }{1 -\theta}  \| \U_0 -\Y \|_F \right)^j, \nonumber  
\end{eqnarray}
where (a) uses (\ref{eq:Dlresnet_M}) and $\|\tilde{\W}_t^i\| \leq 1 + R^{res}$
for bounding a $j \geq 2$ higher-order terms like
$\beta \tilde{\W}^{L:k_j+1}_{t} \cdot  \M_{t}^{k_j} \tilde{\W}^{k_j-1:k_{j-1}+1}_{t} 
\cdot \M_{t}^{k_{j-1}} \cdots   \M_{t}^{k_{1}} \cdot \tilde{\W}^{k_1-1:1}_{t} $, where $1 \leq k_1 < \cdots < k_j \leq L$
 and (b) uses that ${L \choose j  } \leq \frac{L^j}{j!} $.

Then we turn to bound 
$L\frac{108\|\A\| \|\B\| \|\X\|\sqrt{\kappa}}{1+R^{res}} \eta
  \frac{ \theta^{t} }{1 -\theta}  \| \U_0 -\Y \|_F$ in the sum above, it has
\begin{eqnarray}
\label{resnets: inner}
L\frac{108\|\A\| \|\B\| \|\X\|\sqrt{\kappa}}{1+R} \eta
  \frac{ \theta^{t} }{1 -\theta}  \| \U_0 -\Y \|_F
 &
\overset{(a)}{\leq}& \frac{145}{1+R}\frac{L\|\A\| \|\B\| \|\X\|\sqrt{\eta \kappa}}{\sqrt{\lambda_{min}}}\|\U_0 -\Y\|_F \nonumber\\
&\leq& \frac{127\sqrt{\kappa}}{(1+R) \alpha\gamma m \sigma_{min}(\X)}\|\U_0 -\Y\|_F \leq 0.5,\nonumber\\  
\end{eqnarray}
where (a) uses $\theta = 1 - \frac{1}{2\sqrt{\kappa}} \leq 1 - \sqrt{\frac{\eta\lambda_{min}}{2}}$, (b) uses $\eta=\frac{1}{2L\|\A\|^2\|\B\|^2\|\X\|^2}$ and $\lambda_{min} = (0.9)^4 L \alpha^2\gamma^2 m^2 \sigma^2_{min}(\X)$, (c) uses Lemma~\ref{lem:DLresnet_init} and $m \geq C \cdot \max\{\frac{d_y\kappa \log(2n/\delta)\|\X\|_F^2}{\sigma_{min}^2(\X)}, \frac{\sqrt{\kappa}\|\W^*\| \|\X\|_F}{\alpha\gamma \sigma_{min}(\X)}\}$ for a sufficient large constant $C >0$.
Combining the above results, we have 
\begin{eqnarray} \label{resnets:phi}
\| \phib_t \| &  =& \| \B   \Phi_t \A\X \|_F
\nonumber\\& \overset{(a)}{\leq}&  1.001 \|\A\|\|\B\| \| \X \| 
\left( \eta L
\frac{108\|\A\| \|\B\| \|\X\|\sqrt{\kappa}}{1+R}  \frac{ \theta^{t} }{1 -\theta}   \| \U_0 -\Y \|_F \right)^2  
\sum_{j=2}^{L} 
\left(  0.5 \right)^{j-2}    
\nonumber\\ & \overset{(b)}{\leq}&  \frac{11676\kappa^2}{ \|\A\| \|\B\| \|\X\|}  
\left( 
 { \theta^{t} }  \| \U_0 -\Y \|_F \right)^2 \nonumber \\  
& \overset{(c)}{\leq}&   \frac{1}{180\sqrt{\kappa}} \theta^{2t} \| \U_0 -\Y \|_F,
\end{eqnarray}
where (a) uses (\ref{resnets: inner}), (b) uses $\eta = \frac{1}{2L\|\A\|^2\|\B\|^2\|\X\|^2}$ and $\theta = 1 - \frac{1}{2\sqrt{\kappa}}$, (c) uses Lemma~~\ref{lem:DLresnet_init} and $m \geq C \cdot \max\{\frac{d_y\kappa^4 \log(2n/\delta)\|\X\|_F^2}{\sigma_{min}^2(\X)}, \frac{\kappa^{2}\|\W^*\| \|\X\|_F}{\alpha\gamma \sigma_{min}(\X)}\}$ for a sufficiently large constant $C >0$.

Then we turn to analyze the bound of $\| \psib_t \|$.
We need to derive the Frobenius norm of
$\B \left( (L-1) \beta \tilde{\W}^{L:1}_{t} + \beta  \tilde{\W}^{L:1}_{t-1}
- \beta \sum_{l=1}^L \tilde{\W}^{L:l+1}_t \tilde{\W}^{l}_{t-1} \tilde{\W}^{l-1:1}_{t} \right) \A\X$ and\\
 $ \eta\beta \B (\sum_{l=1}^L \tilde{\W}^{L:l+1}_{t} \frac{ \partial \ell(\W^{L:1}_{t-1})}{ \partial \W^{l}_{t-1} } \tilde{\W}^{l-1:1}_{t} - \sum_{l=1}^L \tilde{\W}^{L:l+1}_{t-1} \frac{ \partial \ell(\W^{L:1}_{t-1})}{ \partial \W^{l}_{t-1} } \tilde{\W}^{l-1:1}_{t-1}) \A\X$.
The first term can be rewritten as
\begin{eqnarray} \label{eq:Dlresnetimport}
&& \underbrace{\B   \beta (L-1)  \cdot \Pi_{l=1}^L \left( \tilde{\W}^{l}_{t-1} + \M_{t-1,l} \right) \A\X }_{\text{first term} } + \underbrace{ B\beta  \tilde{\W}^{L:1}_{t-1} \A\X }_{\text{second term}} 
\nonumber\\&&\underbrace{
- \B \beta \sum_{l=1}^L \Pi_{i=l+1}^L \left( \tilde{\W}^{i}_{t-1} + \M_{t-1,i} \right)  \tilde{\W}^{l}_{t-1} \Pi_{j=1}^{l-1} \left( \tilde{\W}^{j}_{t-1} + \M_{t-1,j} \right)  \A\X }_{\text{third term}},\nonumber
\end{eqnarray}
which can be further rewritten as $\E_0 +  \E_1 +  \E_2 + \dots +  \E_L$ for some matrices $\E_0,\dots, \E_L \in \reals^{d_y \times n}$, where $\E_i$ is composed of the multiplication of $i$ momentum terms.
Specifically, we have
\begin{equation}
\begin{split}
\E_0 & = \underbrace{ \B(L-1) \beta \tilde{\W}^{L:1}_{t-1} \A\X }_{ \text{due to the first term} }+ 
 \underbrace{  \B\beta  \tilde{\W}^{L:1}_{t-1} \A\X }_{ \text{due to the second term} }
 \underbrace{
- \B\beta L \tilde{\W}^{L:1}_{t-1} \A\X }_{ \text{due to the third term} } = 0
\\
\E_1 & = \underbrace{ - \B(L-1) \beta \sum_{l=1}^L \tilde{\W}^{L:l+1}_{t-1}  \M_{t-1}^{l} \tilde{\W}^{l-1:1}_{t-1} \A\X }_{ \text{due to the first term} }
+ \underbrace{ \B\beta \sum_{l=1}^L \sum_{k \neq l}
\tilde{\W}^{L:k+1}_{t-1}  \M_{t-1}^{k} \tilde{\W}^{k-1:1}_{t-1} \A\X}_{ \text{due to the third term} }
= 0.
\end{split}
\end{equation}
So what remains on (\ref{eq:Dlresnetimport}) are all the higher-order momentum terms, i.e. those with $ \M_{t-1}^{i}$ and  $ \M_{t-1}^{j}$, $\forall i \neq j$ or higher.

To continue, observe that for a fixed $(i,j)$, $i < j$, 
the second-order term $\E_2$ that involves $ \M_{t-1}^{i}$ and $\M_{t-1}^{j}$ on 
(\ref{eq:Dlresnetimport}) is with coefficient
$\beta$,
because the first term on (\ref{eq:Dlresnetimport})
contributes to
$(L-1) \beta $, while the third term on (\ref{eq:Dlresnetimport})
contributes to
$-(L-2) \beta$. 
Furthermore,
for a fixed $(i,j,k)$, $i < j < k$, 
the third-order term that involves $\M_{t-1}^{i}$, $\M_{t-1}^{j}$, and 
$ \M_{t-1}^{k}$
on (\ref{eq:Dlresnetimport}) is with coefficient $-2 \beta$,
as the first term on (\ref{eq:Dlresnetimport})
contributes to
$(L-1) \beta $, while the third term on (\ref{eq:Dlresnetimport})
contributes to
$-(L-3) \beta$. 
Similarly, for a $p$-order term, the coefficient is 
$-(p-1) \beta $.

Combining all the pieces together, we have
\begin{eqnarray} \label{eq:qqq1}
&&\!\!\!\!\!\!\!\!  \| \B \left( (L-1) \beta \tilde{\W}^{L:1}_{t} + \beta  \tilde{\W}^{L:1}_{t-1}
- \beta \sum_{l=1}^L \tilde{\W}^{L:l+1}_t \tilde{\W}^{l}_{t-1} \tilde{\W}^{l-1:1}_{t} \right) \A\X \|_F
\nonumber\\ &
\overset{(a)}{\leq}& \!\!\!\!\!\!  \beta\|\A\|\|\B\| \|\X\| \sum_{j=2}^L \left(j-1\right) {L \choose j} 
\left( \eta 
108\|\A\| \|\B\| \|\X\|\sqrt{\kappa}  \frac{ \theta^{t-1} }{1 -\theta}     \| \U_0 -\Y \|_F  \right)^j (1+R^{res})^{L-j} 
\nonumber\\ &
\overset{(b)}{\leq}& \beta\|\A\|\|\B\| \|\X\| \sum_{j=2}^L L^j 
\left( \eta
108\|\A\| \|\B\| \|\X\|\sqrt{\kappa}  \frac{ \theta^{t-1} }{1 -\theta}     \| \U_0 -\Y \|_F\right)^j  (1+R^{res})^{L-j}
\nonumber\\ &
\leq& 1.001\beta\|\A\|\|\B\| \|\X\| \sum_{j=2}^L  
\left( 
\frac{L\eta 108\|\A\| \|\B\| \|\X\| \sqrt{\kappa}}{1+R}  \frac{ \theta^{t-1} }{1 -\theta}  \| \U_0 -\Y \|_F\right)^j,   
\end{eqnarray}
where (a) uses (\ref{eq:Dlresnet_M}) and   higher-order terms for any $j\geq 2$ have the form as
$\B\beta (j-1) (-1)^{j} \tilde{\W}^{L:k_j+1}_{t-1} \cdot  \M_{t-1}^{k_j} \tilde{\W}^{k_j-1:k_{j-1}+1}_{t-1} 
\cdot \M_{t-1}^{k_{j-1}}\cdots   \M_{t-1}^{k_{1}} \cdot \tilde{\W}^{k_1-1:1}_{t-1} \A\X$, where $1 \leq k_1 < \cdots < k_j \leq L$
 and (b) uses that ${L \choose j  } \leq \frac{L^j}{j!} $

For the term
$\frac{\eta L 108\|\A\| \|\B\| \|\X\| \sqrt{\kappa}}{1+R^{res}}  \frac{ \theta^{t-1} }{1 -\theta}    \| \U_0 -\Y \|_F$ in the sum above, it follows a similar analysis as (\ref{resnets: inner}) to derive its bound as
\begin{equation} \label{resnets: inner2}
\begin{aligned}
\frac{\eta L 108\|\A\| \|\B\| \|\X\|\sqrt{\kappa}}{1+R^{res}}  \frac{ \theta^{t-1} }{1 -\theta}     \| \U_0 -\Y \|_F
 &\leq 0.5,  
\end{aligned}
\end{equation}
with $m \geq C \cdot \max\{\frac{d_y\kappa \log(2n/\delta)\|\X\|_F^2}{\sigma_{min}^2(\X)}, \frac{\sqrt{\kappa}\|\W^*\| \|\X\|_F}{\alpha\gamma \sigma_{min}(\X)}\}$ for a sufficent large constant $C >0$.
Combining (\ref{eq:qqq1}) and (\ref{resnets: inner2}), it has 
\begin{eqnarray} \label{eq:psipart1}
 &&\| \B \left( (L-1) \beta \tilde{\W}^{L:1}_{t} + \beta  \tilde{\W}^{L:1}_{t-1}
- \beta \sum_{l=1}^L \tilde{\W}^{L:l+1}_t \tilde{\W}^{l}_{t-1} \tilde{\W}^{l-1:1}_{t} \right) \A\X \|_F \nonumber\\ 
& \leq &\frac{11676\kappa^2}{\|\A\| \|\B\| \|\X\|}  
\left( 
  \theta^{t-1}\| \U_0 -\Y \|_F \right)^2\nonumber\\
 &
\overset{(a)}{\leq}& \frac{1}{180\sqrt{\kappa}} \theta^{2t-2}\|\U_0 -\Y\|_F 
 \overset{(b)}{\leq} \frac{1}{45\sqrt{\kappa}} \theta^{2t}\|\U_0 -\Y\|_F,
\end{eqnarray}
where (a) uses $m \geq C \cdot \max\{\frac{d_y\kappa^4 \log(2n/\delta)\|\X\|_F^2}{\sigma_{min}^2(\X)}, \frac{\kappa^{2}\|\W^*\| \|\X\|_F}{\alpha\gamma \sigma_{min}(\X)}\}$ for a sufficiently large constant $C >0$, (b) uses $\theta = 1 - \frac{1}{2\sqrt{\kappa}} \geq 1/2$.

Then we turn to bound $ \eta\beta \B \left(\sum_{l=1}^L \tilde{\W}^{L:l+1}_{t} \frac{ \partial \ell(\W^{L:1}_{t-1})}{ \partial \W^{l}_{t-1} } \tilde{\W}^{l-1:1}_{t} - \sum_{l=1}^L \tilde{\W}^{L:l+1}_{t-1} \frac{ \partial \ell(\W^{L:1}_{t-1})}{ \partial \W^{l}_{t-1} } \tilde{\W}^{l-1:1}_{t-1} \right) \A\X $, it has
\begin{eqnarray}
\label{resnets: psipart12}
&& \eta\beta \|\B \left(\sum_{l=1}^L \tilde{\W}^{L:l+1}_{t} \frac{ \partial \ell(\W^{L:1}_{t-1})}{ \partial \W^{l}_{t-1} } \tilde{\W}^{l-1:1}_{t} - \sum_{l=1}^L \tilde{\W}^{L:l+1}_{t-1} \frac{ \partial \ell(\W^{L:1}_{t-1})}{ \partial \W^{l}_{t-1} } \tilde{\W}^{l-1:1}_{t-1} \right) \A\X\|_F \nonumber\\
&\leq&   \eta\beta \sum_{l=1}^L ( \|\underbrace{\B(\tilde{\W}^{L:l+1}_{t}- \tilde{\W}^{L:l+1}_{t-1})\frac{ \partial \ell(\W^{L:1}_{t-1})}{ \partial \W^{l}_{t-1} } \tilde{\W}^{l-1:1}_{t}\A\X\|}_{\text{first term}} \nonumber \\
&& \quad \quad \quad +\underbrace{\|\B\tilde{\W}^{L:l+1}_{t-1} \frac{ \partial \ell(\W^{L:1}_{t-1})}{ \partial \W^{l}_{t-1} } ( \tilde{\W}^{l-1:1}_{t}-  \tilde{\W}^{l-1:1}_{t-1})\A\X\|}_{\text{second term}}  ).
\end{eqnarray}
For the first term of the above formulation, it has
\begin{eqnarray}
\!\!\!\!\|\B(\tilde{\W}^{L:l+1}_{t} \!-\! \tilde{\W}^{L:l+1}_{t-1})\frac{ \partial \ell(\W^{L:1}_{t-1})}{ \partial \W^{l}_{t-1} } \tilde{\W}^{l-1:1}_{t}\A\X\|_F \leq  \|\B\|\|\tilde{\W}^{L:l+1}_{t} \!-\! \tilde{\W}^{L:l+1}_{t-1}\| \|\frac{ \partial \ell(\W^{L:1}_{t-1})}{ \partial \W^{l}_{t-1} }\|_F \|\tilde{\W}^{l-1:1}_{t}\A\X\|. \nonumber
\end{eqnarray}
It is noted that
\begin{eqnarray}
\label{resnets: cum}
	\|\tilde{\W}_t^{j:i} - \tilde{\W}_{t-1}^{j:i}\| &=& \|\tilde{\W}_t^{j:i} - \tilde{\W}_{t-1}^{j:i}\| \nonumber\\
	&\leq& \|\Pi_{l=i}^j(\tilde{\W}_{t-1}^{l} + \M_{t-1, l}) - \tilde{\W}_{t-1}^{j:i}\| \nonumber\\
	&\leq& \sum_{k=1}^{j-i+1} {j-i+1 \choose k}(1+R^{res})^{j-i+1-k}(108\|\A\| \|\B\| \|\X\|\sqrt{\kappa}  \frac{ \theta^{t-1} }{1 -\theta} \eta \| \U_0 -\Y \|_F)^k \nonumber\\
	&\leq&\!\!\!\! (1+R^{res})^{j-i+1} \sum_{k=1}^{j-i+1} (108\|\A\| \|\B\| \|\X\|\sqrt{\kappa} \frac{j-i+1}{1+R^{res}} \frac{ \theta^{t-1} }{1 -\theta} \eta \| \U_0 -\Y \|_F)^k.
\end{eqnarray}
Thus, it has
\begin{eqnarray}
\label{resnets: psipart121}
	&&\eta\beta\sum_{l=1}^L \|\B(\tilde{\W}^{L:l+1}_{t}- \tilde{\W}^{L:l+1}_{t-1})\frac{ \partial \ell(\W^{L:1}_{t-1})}{ \partial \W^{l}_{t-1} } \tilde{\W}^{l-1:1}_{t}\A\X\|_F \nonumber\\
& \leq& \eta\beta \sum_{l=1}^L\|\B\|\|\tilde{\W}^{L:l+1}_{t}- \tilde{\W}^{L:l+1}_{t-1}\| \|\frac{ \partial \ell(\W^{L:1}_{t-1})}{ \partial \W^{l}_{t-1} }\|_F \|\tilde{\W}^{l-1:1}_{t}\A\X\| \nonumber\\
	&\overset{(a)}{\leq}& 36 \sqrt{\kappa}\eta \beta (\|\A\|\|\B\|\|\X\|)^2 \theta^{t-1} \|\U_0 -\Y\|_F (1+R^{res})^{L-1} \nonumber \\
	&&\cdot\sum_{l=1}^L\sum_{k=1}^{L-l}(108\|\A\| \|\B\| \|\X\|\sqrt{\kappa} \frac{L-l}{1+R^{res}} \frac{ \theta^{t-1} }{1 -\theta} \eta \| \U_0 -\Y \|_F)^k \nonumber\\
	&\overset{(b)}{\leq}& 36\sqrt{\kappa} \eta \beta (\|\A\|\|\B\|\|\X\|)^2 \theta^{t-1} \|\U_0 -\Y\|_F (1+R^{res})^{L-1} \nonumber \\
	&&\cdot 108\|\A\| \|\B\| \|\X\|\sqrt{\kappa}\frac{L}{1+R^{res}} \frac{ \theta^{t-1} }{1 -\theta} \eta \| \U_0 -\Y \|_F \sum_{l=1}^L\sum_{k=1}^{L-l}0.5^{k-1} \nonumber\\
	&\overset{(c)}{\leq}& \frac{3920\kappa^{3/2}}{\|\A\| \|\B\| \|\X\| } (\theta^{t-1} \|\U_0 -\Y\|_F)^2 \nonumber\\
	&\overset{(d)}{\leq}&   \frac{1}{360\sqrt{\kappa}} \theta^{2t-2}\|\U_0 -\Y\|_F \leq \frac{1}{90\sqrt{\kappa}} \theta^{2t}\|\U_0 -\Y\|_F,
\end{eqnarray}
where (a) uses (\ref{eq:DLresnetnorm-linear}) and (\ref{resnets: cum}), (b) uses 
 and (\ref{resnets: inner2}), 
(c) uses $\beta \leq 1$, $\eta = \frac{1}{2L\|\A\|^2\|\B\|^2\|\X\|^2}$ and $(1+R^{res})^{L-1} \leq (1+R^{res})^L \leq 1.001$, (d) uses $m \geq C \cdot \max\{\frac{d_y\kappa^3 \log(2n/\delta)\|\X\|_F^2}{\sigma_{min}^2(\X)}, \frac{\kappa^{3/2}\|\W^*\| \|\X\|_F}{\alpha\gamma \sigma_{min}(\X)}\}$

For the second part of (\ref{resnets: psipart12}), it has the same bound as
\begin{eqnarray}
\label{resnets: psipart122}
&& \eta\beta\sum_{l=1}^L \|\B\tilde{\W}^{L:l+1}_{t-1} \frac{ \partial \ell(\W^{L:1}_{t-1})}{ \partial \W^{l}_{t-1} } ( \tilde{\W}^{l-1:1}_{t}-  \tilde{\W}^{l-1:1}_{t-1})\A\X\| \nonumber \\
&\leq& 36\sqrt{\kappa} \eta \beta (\|\A\|\|\B\|\|\X\|)^2 \theta^{t-1}  \|\U_0 -\Y\|_F (1+R^{res})^{L-1} \nonumber\\
&& \cdot \sum_{l=1}^L\sum_{k=1}^{L-l}(108\|\A\| \|\B\| \|\X\|\sqrt{\kappa} \frac{L-l}{1+R^{res}} \frac{ \theta^{t-1} }{1 -\theta} \eta \| \U_0 -\Y \|_F)^k \nonumber\\
&\leq& \frac{1}{90\sqrt{\kappa}} \theta^{2t}\|\U_0 -\Y\|_F.
\end{eqnarray}

Combining (\ref{eq:psipart1}), (\ref{resnets: psipart121}) and (\ref{resnets: psipart122}), it has
\begin{equation}
\label{resnets: psibound}
\begin{aligned}
 \psib_t&= 
\v\left(  \B ( (L-1) \beta \tilde{\W}^{L:1}_{t} + \beta  \tilde{\W}^{L:1}_{t-1}
- \beta \sum_{l=1}^L \tilde{\W}^{L:l+1}_t \tilde{\W}^{l}_{t-1} \tilde{\W}^{l-1:1}_{t} ) \A\X \right)\\
& + \v\left(\eta\beta \B (\sum_{l=1}^L \tilde{\W}^{L:l+1}_{t} \frac{ \partial \ell(\W^{L:1}_{t-1})}{ \partial \W^{l}_{t-1} } \tilde{\W}^{l-1:1}_{t} - \sum_{l=1}^L \tilde{\W}^{L:l+1}_{t-1} \frac{ \partial \ell(\W^{L:1}_{t-1})}{ \partial \W^{l}_{t-1} } \tilde{\W}^{l-1:1}_{t-1}) \A\X\right) \\
\|\psib_t\|_F &\leq   \frac{2}{45\sqrt{\kappa}} \theta^{2t}\|\U_0 -\Y\|_F. 
\end{aligned}
\end{equation}

Now let us switch to bound $\| \iotab_t \| \leq \|\eta(1+\beta) (\H_t - \H_0) \xib_t\| + \|\eta\beta(\H_{t-1} - \H_0)\xib_{t-1}\|$.
It has
\begin{equation} 
\begin{aligned} \label{eq:jj1}
& \| \eta (1+\beta) (\H_t - \H_0) \xib_t \|
\\ & =
\eta(1+\beta)
\| \sum_{l=1}^L \B\tilde{\W}^{L:l+1}_t (\B\tilde{\W}^{L:l+1}_t)^\top ( \U_t -\Y) (\tilde{\W}^{l-1:1}_t \A\X)^\top \tilde{\W}^{l-1:1}_t \A\X  -
\sum_{l=1}^L \B\B^\top ( \U_t -\Y) (A \X)^\top A \X\|_F
\\ & \leq 
\eta (1+\beta) 
\sum_{l=1}^L  
\| \B\tilde{\W}^{L:l+1}_t (\B\tilde{\W}^{L:l+1}_t)^\top ( \U_t -\Y) (\tilde{\W}^{l-1:1}_t \A\X)^\top \tilde{\W}^{l-1:1}_t \A\X  -
\sum_{l=1}^L \B\B^\top ( \U_t -\Y) (A \X)^\top A \X\|_F
\\ & \leq 
\eta (1+\beta) 
\sum_{l=1}^L  \big( 
\underbrace{ 
\| \left(\B\tilde{\W}^{L:l+1}_t (\B\tilde{\W}^{L:l+1}_t)^\top - \B\B^{\top}\right)   ( \U_t -\Y) (\tilde{\W}^{l-1:1}_t \A\X)^\top \tilde{\W}^{l-1:1}_t  \A\X \|_F }_{\text{ first term} }
 \\
 &+ 
\underbrace{ 
\| \B\B^\top  ( \U_t -\Y) \left((\tilde{\W}^{l-1:1}_t \A\X)^\top \tilde{\W}^{l-1:1}_t  \A\X -( \A \X)^\top \A\X \right)   \|_F \big)
}_{\text{ second term} }.
\end{aligned}
\end{equation}

Denote $\Delta_t^{L:l+1} = \B\tilde{\W}^{L:l+1}_t - \B\tilde{\W}^{L:l+1}_0 =  \B\tilde{\W}^{L:l+1}_t - \B$, it has
\begin{eqnarray}
	\|\B\tilde{\W}^{L:l+1}_t (\B\tilde{\W}^{L:l+1}_t)^\top - \B\B^{\top}\| &\leq&  \|(\B+\Delta_t^{L:l+1})(\B+\Delta_t^{L:l+1})^{\top} - \B\B^{\top}\| \nonumber\\
	&\leq& 2\|\B\|\|\Delta_t^{L:l+1}\| + \|\Delta_t^{L:l+1}\|^2 
	\leq \frac{\|\B\|^2}{579\kappa}, \nonumber
\end{eqnarray}
where the last inequality uses
\begin{eqnarray}
	\|\Delta_t^{j:i}\| \leq \|\B\tilde{\W}^{j:i}_t - \B \| &\leq& \|\B\|\|\Pi_{l=i}^j(\W_t^l+\I) -\I \| \leq \|\B\|\sum_{l=1}^{j-i+1} {j-i+1 \choose l} (R^{res})^l \nonumber\\                                                                                                                                                                                                                                                                                                                                                                                                                                                                                                                         &\leq& \|\B\|[(1+R^{res})^{j-i+1} - 1] \leq \frac{\|\B\|}{1160\kappa}.\nonumber
\end{eqnarray}
Similarly, we can derive 
\begin{eqnarray}
	\|(\tilde{\W}^{l-1:1}_t \A\X)^\top (\tilde{\W}^{l-1:1}_t  \A\X) -( \A \X)^\top \A\X\| \leq \frac{\|\A\|^2\|\X\|^2}{579\kappa}.
\end{eqnarray}

Therefore, the first part of (\ref{eq:jj1}) has the bound as
\begin{eqnarray}
	&&\eta (1+\beta) 
\sum_{l=1}^L  
\| \left(\B\tilde{\W}^{L:l+1}_t (\B\tilde{\W}^{L:l+1}_t)^\top - \B\B^{\top}\right)   ( \U_t -\Y) (\tilde{\W}^{l-1:1}_t \A\X)^\top \tilde{\W}^{l-1:1}_t  \A\X \|_F \nonumber \\
& \leq& \eta (1+\beta) \sum_{l=1}^L \frac{\|\B\|^2}{579\kappa}\sqrt{2}\theta^t \nu C_0 \|\U_0 -\Y\|(1+R^{res})^{2l-2}\|\A\|^2\|\X\|^2 \nonumber \\
&\overset{(a)}{\leq} &\frac{1}{17\sqrt{\kappa}} \theta^t \nu C_0 \|\U_0 -\Y\|, \nonumber
\end{eqnarray}
where (a) uses $\eta = \frac{1}{2L\|\A\|^2\|\B\|^2\|\X\|^2}$ and $(1+R^{res})^L \leq 1.001$.

\begin{eqnarray}
	&&\eta(1+\beta)\sum_{l=1}^L\| \B\B^\top  ( \U_t -\Y) \left((\W^{l-1:1}_t \A\X)^\top \W^{l-1:1}_t  \A\X -( \A \X)^\top \A\X \right)   \|_F \nonumber \\
	&\leq&\eta(1+\beta)\sum_{l=1}^L\|\B\|^2\|\U_t -\Y\|_F \frac{\|\A\|^2\|\X\|^2}{579\kappa} \leq \frac{1}{17\sqrt{\kappa}} \theta^t   \|\U_0 -\Y\|_F. \nonumber
\end{eqnarray}

Thus $\| \eta (1+\beta) (\H_t - \H_0) \xib_t \| \leq \frac{2}{17\sqrt{\kappa}} \theta^t \|\U_0 -\Y\|_F$.
The bound of $\|\eta\beta(\H_{t-1} - \H_0)\xib_{t-1}\| \leq \frac{1}{17\sqrt{\kappa}} \theta^{t+1} \|\U_0 -\Y\|_F$ can be derived with a similar way.

Combining the above bounds and $\theta \leq 1$, it has
\begin{equation} \label{resnets:iota}
\begin{split}
\| \iotab_t \| \leq \frac{3}{17\sqrt{\kappa}} \theta^{t} \|\U_0 -\Y\|_F.
\end{split}
\end{equation}

Now we have (\ref{resnets:phi}), (\ref{resnets: psibound}), and  (\ref{resnets:iota}),  
which leads to
\begin{eqnarray}
\| \varphib_t \| & \leq& 
\| \phib_t \| + \|\psib_t \| + \| \iotab_t \| \nonumber\\
&\leq& \frac{1}{180\sqrt{\kappa}} \theta^{2t}\|\U_0 -\Y\|_F + \frac{2}{45\sqrt{\kappa}} \theta^{2t}\|\U_0 -\Y\|_F +\frac{3}{17\sqrt{\kappa}} \theta^{t} \|\U_0 -\Y\|_F \nonumber\\
&\leq& \frac{1}{30\sqrt{\kappa}} \theta^{2t}\|\U_0 -\Y\|_F +\frac{3}{17\sqrt{\kappa}} \theta^{t} \|\U_0 -\Y\|_F. \nonumber
\end{eqnarray}
In addition, it has $m \geq C_1\max\{{d_y r \kappa^5 \log(2n/\delta)}, \frac{\sqrt{r}\kappa^{2.5}\|\W^*\| }{\alpha\gamma }\} \geq C_2 \cdot \max\{\frac{d_y\kappa^4 \log(2n/\delta)\|\X\|_F^2}{\sigma_{min}^2(\X)}, \frac{\kappa^{2}\|\W^*\| \|\X\|_F}{\alpha\gamma \sigma_{min}(\X)}\}$ for some sufficiently large constant $C_1, C_2 >0$  using $\|\X\|_F \leq \sqrt{r}\|\X\|$.
\end{proof}

\begin{lemma}~\label{deepres: lemma_distance}
Following the setting as Lemma~\ref{deepres: lemm2},
for any $s \leq t$, assume the residual dynamics satisfies
$\textstyle \left\|
\begin{bmatrix}
\xib_{s} \\
\xib_{s-1} 
\end{bmatrix}
\right\| \leq \theta^{s} 
\cdot 24\sqrt{\kappa}
\left\|
 \begin{bmatrix}
\xib_{0} \\
\xib_{-1} 
\end{bmatrix}
\right\|,
$ 
then 
\[
\| \W^{l}_t - \W^{l}_0 \|_F \leq R^{res} = 
\frac{1}{2000L\kappa}.
\]
\end{lemma}

\begin{proof}
We have 
\begin{equation}
\begin{split}
 \| \W^{l}_{t+1} - \W^{l}_0 \|_F 
& \overset{(a)}{\leq} 
 \sum_{s=0}^t \|  \M_{s,l} \|_F \\ 
&\overset{(b)}{\leq} 3.003\|\A\|\|\B\|\|\X\|\eta 24\sqrt{\kappa}\| \U_0 -\Y \|_F \sum_{s=0}^t \frac{\theta^s}{1-\theta} 
\\ &
\overset{(c)}{\leq} \frac{145 \kappa^{3/2}}{L\|\A\|\|\B\|\|\X\|}  \| \U_0 -\Y \|_F  
\\&
\overset{(d)}{\leq} \frac{1}{2000L\kappa},
\end{split}
\end{equation}
where (a) uses the update rule of momentum $\W^{l}_{t+1} - \W^{l}_t = - \eta \M_t^l$, where $\M_t^l$,
(b) uses the bound of $\M_t^l$ in (\ref{eq:Dlresnet_M}),
(c) uses $\frac{1}{(1-\theta)^2} = 4\kappa$ and $\eta  = \frac{1}{2L\|\A\|^2\|\B\|^2\|\X\|^2}$,
(d) uses Lemma~\ref{lem:DLresnet_init} and   $m \geq C \cdot \max\{\frac{d_y\kappa^4 \log(2n/\delta)\|\X\|_F^2}{\sigma_{min}^2(\X)}, \frac{\kappa^{2}\|\W^*\| \|\X\|_F}{\alpha\gamma \sigma_{min}(\X)}\}$ for a sufficient large constant $C > 0$.

\end{proof}

\subsection{Proof of Theorem~\ref{thm:resnet}}
\begin{proof}

We prove the theorem by induction.
The base case $s = 0$ holds.
Assume $
\left\|
\begin{bmatrix}
\xib_{s} \\
\xib_{s-1} 
\end{bmatrix}
\right\|
\leq \theta^{s} 24\sqrt{\kappa}  \left\|
\begin{bmatrix}
\xib_{0} \\
\xib_{-1} 
\end{bmatrix}
\right\|
$ holds for $s \leq t-1$.

Based on Lemma~\ref{deepres: lemm1}, it is noted that
\begin{eqnarray}
\begin{bmatrix}
\xib_{t} \\
\xib_{t-1} 
\end{bmatrix}
= 
\G
\begin{bmatrix}
\xib_{t-1} \\
\xib_{t-2} 
\end{bmatrix}
+
\begin{bmatrix}
\varphib_{t-1} \\ \0_{d_y n}
\end{bmatrix}, \nonumber
\end{eqnarray}
where $\G = \begin{bmatrix}
(1+\beta)(\I_{d_y n} - \eta \H_0^{res}) & \beta ( -\I_{d_y n} + \eta \H_{0}^{res}  )   \\
\I_{d_y n} & \0_{d_y n} 
\end{bmatrix}$.
By recursively using above equation, it has
\begin{eqnarray}
\label{deepres: theorem_1}
	\begin{bmatrix}
\xib_{t} \\
\xib_{t-1} 
\end{bmatrix}	= \G^t \begin{bmatrix}
\xib_{0} \\
\xib_{-1} 
\end{bmatrix} + \sum_{s=0}^{t-1} \G^{t-s-1} \begin{bmatrix}
\varphib_s \\
\0_{d_y n}
\end{bmatrix}.
\end{eqnarray}
From Lemma~\ref{supportlemma1} and Lemma~\ref{supportlemma2}, it has the bound for the first term on the right hand side of (\ref{deepres: theorem_1}) as
\begin{eqnarray}
\label{deepres: theorem_2}
\left\|\G^t \begin{bmatrix}
\xib_{0} \\
\xib_{-1} 
\end{bmatrix}  \right\| \leq 12\sqrt{\kappa}\rho^t  \left\|\begin{bmatrix}
\xib_{0} \\
\xib_{-1} 
\end{bmatrix} \right\|,
\end{eqnarray}
where $\rho = 1 - \frac{2}{3\sqrt{\kappa}}$.

Applying the inductive hypothesis and Lemma~\ref{deepres: lemma_distance}, it has the upper bound for the distance $\|\W_i^l - \W_0^l\| \leq R^{res}$ for any $i \leq t$ and $l \in [m]$.
In turn, we can bound the second term on the right hand side of (\ref{deepres: theorem_1}) as

\begin{eqnarray} \label{deepres: theorem_3}
 \left\| \sum_{s=0}^{t-1} \G^{t-1-s} \begin{bmatrix}
\varphib_{s} \\ 0 \end{bmatrix} \right\|
& \overset{(a)}{ \leq}& 
\sum_{s=0}^{t-1} 12\sqrt{\kappa}\rho^{t-1-s}  \| \varphib_s \| \nonumber\\
& \overset{(b)}{ \leq} & \sum_{s=0}^{t-1} \rho^{t-1-s}12\sqrt{\kappa}(\frac{1}{30\sqrt{\kappa}} \theta^{2s}\| \U_0 -\Y \|_F + \frac{3}{17\sqrt{\kappa}}  \theta^s \|\U_0 -\Y\|_F)
\nonumber\\ 
& \overset{(c)}{ \leq} &
12\sqrt{\kappa}  \theta^t (\frac{\sqrt{2}}{10} + \frac{9\sqrt{2}}{17})\left\| \begin{bmatrix} \xib_0 \\ \xib_{-1} \end{bmatrix} \right\|\nonumber\\
& \overset{}{ \leq}&
12\sqrt{\kappa}  \theta^{t}    \left\| \begin{bmatrix} \xib_0 \\ \xib_{-1} \end{bmatrix} \right\|,
\end{eqnarray}
where (a) uses Lemma~\ref{supportlemma1} and Lemma~\ref{supportlemma2},
(b) uses the bound of $\|\varphib_{s}\|$ in Lemma~\ref{deepres: lemm2},
 (c) uses $\sum_{s=0}^{t-1} \rho^{t-1-s} \theta^s = \theta^{t-1} \sum_{s=0}^{t-1} \left( \frac{\rho}{\theta}  \right)^{t-1-s} \leq \theta^{t-1} \frac{1-(\rho/\theta)^t}{1-\rho/\theta} \leq 6\sqrt{\kappa}\theta^t $ and $\|\xib_{-1}\| = \|\xib_0\|$.
 
Combining (\ref{deepres: theorem_2}) and (\ref{deepres: theorem_3}), it completes the proof.

\end{proof}

\section{Supporting Lemmas}
\label{supporting}
\begin{lemma}(Lemma 2 in~\cite{DBLP:journals/corr/abs-2107-01832})
\label{supportlemma1}
  Assume $\H \in \mathbb{R}^{n \times n}$ is a symmetry positive definite matrix.
Let{\small{ $\G = \begin{bmatrix}
   (1+\beta)(\I_n-\eta \H) &
  \beta(-\I_n+\eta \H) \\
   \I_n & \textbf{0}_n 
   \end{bmatrix} \in \mathbb{R}^{2n \times 2n}$}}.
   Suppose a sequence of iterates $\{\bm{v}_i\}$ satisfy $\bm{v}_t = \G \bm{v}_{t-1}$ for any $t \leq T$.
   If $\beta$ and $\eta$ are chosen that satisfy $1 > \beta \geq \frac{1-\sqrt{\eta\lambda_{min}(\H)}}{1+\sqrt{\eta\lambda_{min}(\H)}}$ and $0 < \eta \leq 1/\lambda_{max}(\H)$, then it has the bound at any iteration $k \leq T$ as
\begin{equation}
\label{eq:the bound of matrix vector}
\|\bm{v}_k\| \leq C \big(\sqrt{\beta(1-\eta\lambda_{min}(\H))}\big)^k  \|v_0\|, 
\end{equation}
where $ C  = \frac{2\beta(1-\eta\lambda_{min}(\H)) + 2}{\sqrt{\min\{g(\beta, \eta\lambda_{min}(\H)), g(\beta, \eta\lambda_{max}(\H))\}}}$ and the function $g$ is defined as $g(x, y) = 4x(1-y) - [(1+x)(1-y)]^2$.
\end{lemma}

\begin{lemma}(Lemma 3 in \cite{DBLP:journals/corr/abs-2107-01832})
\label{supportlemma2}
Assume $0 < \lambda \leq \lambda_{min}(\H) \leq \lambda_{max}(\H) \leq \lambda_{max}$.
Denote ${\kappa} = \lambda_{max}/\lambda$.
With $\eta = 1/2\lambda_{max}$ and $\beta = \frac{3\sqrt{{\kappa}} - 2}{3\sqrt{{\kappa}} + 2}$, it has
\begin{eqnarray}
\sqrt{\beta(1-\eta\lambda_{min}(\H))} \leq 1 - \frac{2}{3\sqrt{{\kappa}}} ,\;\; C \leq 12\sqrt{{\kappa}}.
\end{eqnarray} 
\end{lemma}

\end{document}